\documentclass[sigconf]{acmart}

\AtBeginDocument{%
  \providecommand\BibTeX{{%
    \normalfont B\kern-0.5em{\scshape i\kern-0.25em b}\kern-0.8em\TeX}}}

\copyrightyear{2023}
\acmYear{2023}
\setcopyright{rightsretained}
\acmConference[KDD '23]{Proceedings of the 29th ACM SIGKDD Conference on Knowledge Discovery and Data Mining}{August 6--10, 2023}{Long Beach, CA, USA}
\acmBooktitle{Proceedings of the 29th ACM SIGKDD Conference on Knowledge Discovery and Data Mining (KDD '23), August 6--10, 2023, Long Beach, CA, USA}
\acmDOI{10.1145/3580305.3599547}
\acmISBN{979-8-4007-0103-0/23/08}

\makeatletter
\gdef\@copyrightpermission{
  \begin{minipage}{0.3\columnwidth}
  \end{minipage}\hfill
  \begin{minipage}{0.7\columnwidth}
   \href{https://creativecommons.org/licenses/by/4.0/}{This work is licensed under a Creative Commons Attribution International 4.0 License.}
  \end{minipage}
  \vspace{5pt}
}
\makeatother

\usepackage{url}

\usepackage[utf8]{inputenc} 
\usepackage[T1]{fontenc}    
\usepackage{hyperref}       
\usepackage{url}            
\usepackage{booktabs}       
\usepackage{amsfonts}       
\usepackage{nicefrac}       
\usepackage{microtype}      
\usepackage{xcolor}         

\usepackage[ruled,vlined,linesnumbered, resetcount]{algorithm2e}
\usepackage{graphicx}
\usepackage{amsmath}

\usepackage{subcaption}
\usepackage{multirow}
\usepackage[normalem]{ulem}
\useunder{\uline}{\ul}{}
\usepackage{textcomp}
\usepackage{footnote}
\usepackage{makecell}
\usepackage{xspace}
\usepackage{enumitem}
\usepackage{mdwlist}
\usepackage{wrapfig,lipsum,booktabs}
\usepackage{amsthm}

\usepackage{pdfrender}
\newcommand*{\boldcheckmark}{%
  \textpdfrender{
    TextRenderingMode=FillStroke,
    LineWidth=.5pt, 
  }{\checkmark}%
}

\newtheorem{observation}{Observation}
\newtheorem{definition}{Definition}
\newcommand{\hide}[1]{}

\newcommand{\model}{\textsc{PROFHiT}\xspace}
\newcommand{\camul}{\textsc{TSFNP}\xspace}
\newcommand{\sdcr}{\textsc{SoftDisCoR}\xspace}
\newcommand{\deepvar}{\textsc{DeepAR}\xspace}
\newcommand{\hiere}{\textsc{HierE2E}\xspace}
\newcommand{\mint}{\textsc{MinT}\xspace}
\newcommand{\erm}{\textsc{ERM}\xspace}

\newcommand{\pembu}{\textsc{PEMBU}\xspace}
\newcommand{\sharq}{\textsc{SHARQ}\xspace}
\newcommand{\labour}{\texttt{Labour}\xspace}
\newcommand{\tourism}{\texttt{Tourism-L}\xspace}
\newcommand{\symp}{\texttt{Flu-Symptoms}\xspace}
\newcommand{\wiki}{\texttt{Wiki}\xspace}
\newcommand{\fbsymp}{\texttt{FB-Survey}\xspace}
\newcommand{\finetune}{\textsc{P-FineTune}\xspace}
\newcommand{\pglobal}{\textsc{P-NoParamShare}\xspace}
\newcommand{\pvar}{\textsc{P-DeepAR}\xspace}
\newcommand{\nocoherent}{\textsc{P-NoConsistency}\xspace}
\newcommand{\norefine}{\textsc{P-NoRefine}\xspace}
\def\basemu{{\mu}}
\def\basesigma{{\sigma}}
\def\refinedmu{{\hat{\mu}}}
\def\refinedsigma{{\hat{\sigma}}}

\ccsdesc[300]{Computing methodologies~Machine learning}

\begin{document}
\title{When Rigidity Hurts: Soft Consistency Regularization for Probabilistic Hierarchical Time Series Forecasting}

\author{%
  Harshavardhan Kamarthi}
\affiliation{
  College of Computing,
  Georgia Institute of Technology
  \country{USA}}
\email{hkamarthi3@gatech.edu}
\author{%
  Lingkai Kong}
\affiliation{
  College of Computing,
  Georgia Institute of Technology
  \country{USA}}
\email{lkkong@gatech.edu}
\author{%
  Alexander Rodr\'iguez}
\affiliation{
  College of Computing,
  Georgia Institute of Technology
  \country{USA}}
\email{arodriguezc@gatech.edu}
\author{%
  Chao Zhang}
\affiliation{
  College of Computing,
  Georgia Institute of Technology
  \country{USA}}
\email{chaozhang@gatech.edu}
\author{%
  B. Aditya Prakash}
\affiliation{
  College of Computing,
  Georgia Institute of Technology
  \country{USA}}
\email{badityap@cc.gatech.edu}


\begin{abstract}

  Probabilistic hierarchical time-series forecasting is an important variant
  of time-series forecasting, where the goal is to model and forecast
  multivariate time-series that have hierarchical relations.
  Previous works assume rigid consistency over the given hierarchies and do not adapt well
  to real-world data that show deviation from this assumption.
  Moreover, recent state-of-art
  neural probabilistic methods also impose hierarchical relations on point
  predictions and samples of the predictive distribution.
  This does not account for
  full forecast distributions being consistent with the hierarchy and leading to poorly calibrated forecasts.
  We close both these gaps and propose \model,
  a probabilistic hierarchical forecasting model that jointly models forecast
  distributions over the entire hierarchy.
  \model (1) uses a flexible probabilistic
  Bayesian approach and
  (2) introduces \emph{soft distributional consistency regularization} that enables end-to-end learning of the entire
  forecast distribution leveraging information from the underlying hierarchy.
  This enables calibrated forecasts as well as
  adaptation to real-life data with varied hierarchical consistency.
  \model provides 41-88\% better performance in accuracy and significantly better
  calibration over a wide range of dataset consistency. Furthermore,
  \model adapts to missing data and
  can provide reliable forecasts even if up to 10\% of input
  time-series data is missing, whereas other methods' performance severely degrades
  by over 70\%.
\end{abstract}

\maketitle

\section{Introduction}

Time-series forecasting is an important problem that impacts decision-making in a wide range of applications.
In many real-world situations, the time-series have inherent hierarchical
relations and structures.
Examples  include forecasting time-series of
employment  \citep{taieb2017coherent} measured at different
geographical scales; epidemic forecasting \citep{reich2019collaborative} at county, state and country, etc.
Given time-series dataset with underlying hierarchical relations, the goal of hierarchical time-series forecasting is to generate an accurate forecast for all time-series leveraging the hierarchical relations between time-series  \citep{hyndman2011optimal}.



\begin{figure}[h]
    \centering
    \includegraphics[width=.97\linewidth]{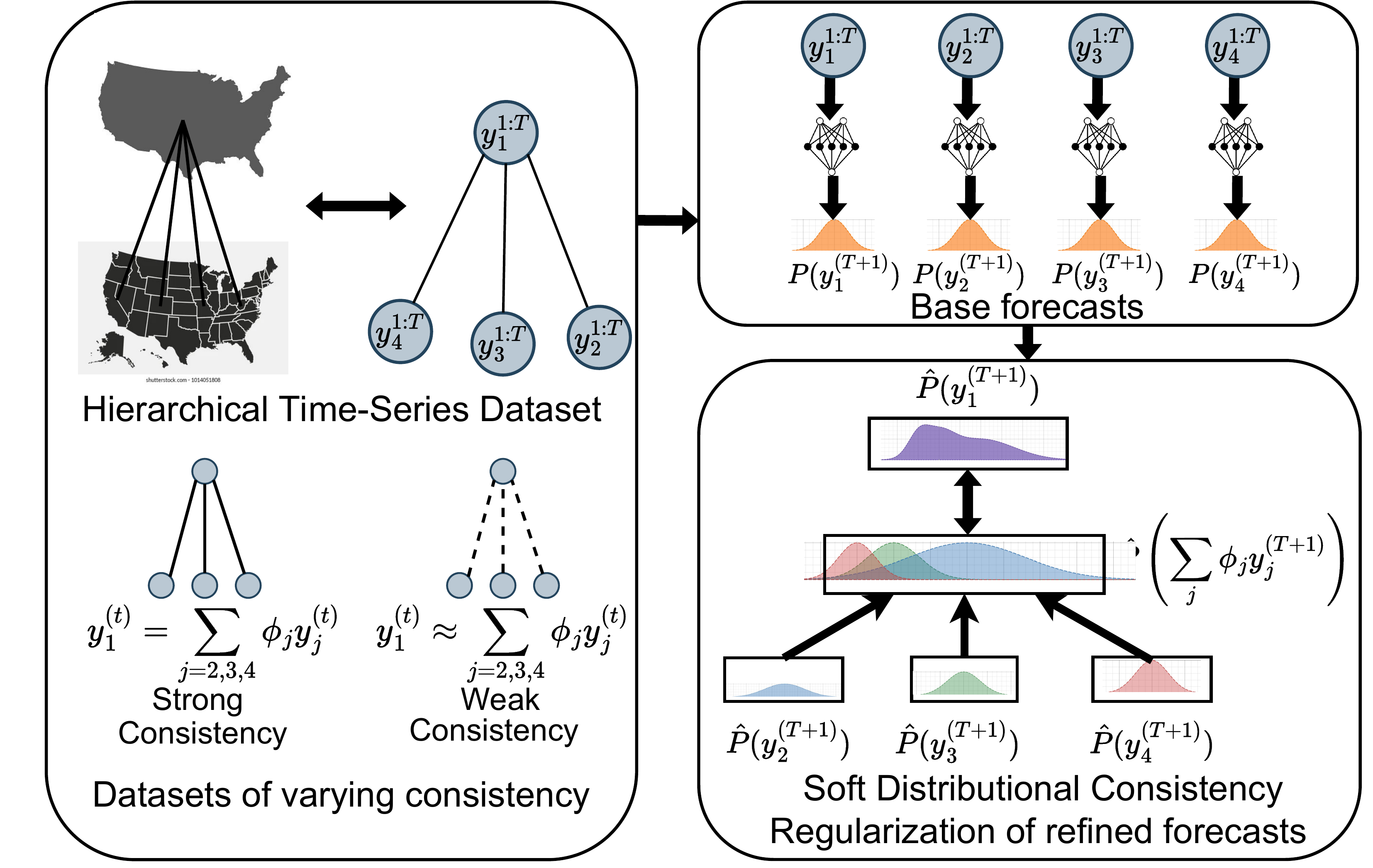}
    \caption{\model learns to produce accurate and calibrated forecasts from datasets of varying consistency
        by leveraging underlying hierarchical relations via Soft Distributional Consistency Regularization
        }
    \label{fig:start}
\end{figure}

Previous hierarchical forecasting methods assume that the dataset is \textit{strongly consistent}:
the time-series values of datasets strictly satisfy
the underlying hierarchical constraints. Therefore, these models usually impose
the generated forecasts to be strongly consistent as well i.e., forecasts strictly satisfy the hierarchical relations of the dataset.
For example, classical
\textit{two-step methods}~\citep{hyndman2018forecasting}
use a bottom-up or top-down approach where all time-series at
a single level of the hierarchy are modeled
independently and the values of other levels are derived using
the aggregation function governing the hierarchy.
In contrast, many real-world
applications have \textit{weakly consistent} datasets, i.e., the data do not follow the strict
constraints of the hierarchy.
Such datasets have an underlying data generation
process that follows a hierarchical set of constraints but may contain some deviations.
These deviations can be caused by factors such as
measurement or reporting error, asynchrony in data aggregation and revision
pipeline, etc, as frequently observed in epidemic forecasting \citep{adhikariKDD2019}. 
Most state-of-the-art methods are
designed for applications having strongly consistent datasets by imposing rigid constraints --- they thus may not adapt to such
deviations and can \textit{provide poor forecasts for application with weakly consistent datasets}

Moreover, previous methods do not focus on providing \textit{calibrated forecasts} with precise uncertainty measures.
Traditional methods focus on point predictions only.
Recent \textit{post-processing methods} \cite{wickramasuriya2019optimal,ben2019regularized,taieb2017coherent}
refine base independent forecast distribution as a post-processing step.
While these methods can be easily applied to forecasts from any model,
this does not enable the models generating the base forecasts to learn from hierarchical relations between time-series of the hierarchy.
\textit{End-to-end learning neural methods} directly leverage hierarchical relations as part of the model architecture~\cite{rangapuram2021end} or learning algorithm~\cite{han2021simultaneously}.
Due to their comprehensive end-to-end approach, they usually outperform post-processing methods by imposing hierarchical constraints on the mean or fixed quantiles of the forecast distributions.
However, these methods do not enforce hierarchical consistency on the full distributions.
Therefore, the \textit{forecasts may not be well-calibrated} \citep{kuleshov2018accurate} i.e., they produce unreliable prediction intervals that may not match observed probabilities from ground truth \citep{fisch2022calibrated}.

\begin{table}[h]
    \caption{Comparison of \model with state-of-the-art methods.}
    \label{tab:comp}
    \centering
    \scalebox{0.80}{
        \begin{tabular}{l|cccc}
                                                                                  & \begin{tabular}[c]{@{}l@{}}Two-step\\ methods \end{tabular} & \begin{tabular}[c]{@{}l@{}}Post-processing\\methods\end{tabular} & \begin{tabular}[c]{@{}l@{}}End-to-end\\neural methods \end{tabular} & \begin{tabular}[c]{@{}l@{}}\textbf{\model}\\(This paper)\end{tabular} \\ \hline
            \begin{tabular}[l]{@{}l@{}}Probabilistic\\Forecasts\end{tabular}      & $\times$                                                    & \checkmark                                                       & \checkmark                                                          & \boldcheckmark                                                        \\\hline
            \begin{tabular}[c]{@{}l@{}}Strong \& Weak\\ Consistency \end{tabular} & $\times$                                                    & $\times$                                                         & $\times$                                                            & \boldcheckmark                                                        \\\hline
            \begin{tabular}[c]{@{}l@{}}Distributional\\ Consistency \end{tabular}   & $\times$                                                    & \checkmark                                                       & $\times$                                                            & \boldcheckmark                                                        \\\hline
            \begin{tabular}[c]{@{}l@{}}End-to-end\\ Learning \end{tabular}        & $\times$                                                    & $\times$                                                         & \checkmark                                                          & \boldcheckmark                                                        \\\hline
        \end{tabular}
    }
\end{table}

In this work, we fill this gap of learning well-calibrated and accurate forecasts for both strong and weakly consistent datasets leveraging underlying hierarchical relations.
We propose \model (\underline{P}robabilistic \underline{Ro}bust \underline{F}orecasting for \underline{Hi}erarchical
\underline{T}ime-series),  a neural probabilistic hierarchical time-series forecasting method that
provides an end-to-end Bayesian approach to model the distributions of
forecasts of all time-series together (see Table \ref{tab:comp} for a comparison).
Specifically, we introduce a novel \textit{Soft Distributional Consistency Regularization (\sdcr)} to tackle the challenge.
First, \sdcr enables \model to leverage hierarchical relations over entire forecast distributions
to generate calibrated forecast distributions by encouraging the forecast distribution of any parent node to be similar to the aggregation of children nodes' forecast distributions (Figure \ref{fig:start}).
Second, since \sdcr is a soft constraint, our model is trained to adapt to datasets with varying
hierarchical consistency that allows the model to trade-off consistency for better accuracy and calibration on weakly consistent datasets.
Our main contributions are:
\begin{enumerate}
      \item \textbf{Accurate and Calibrated Probabilistic Hierarchical Time-Series Forecasting:} We propose \model, a deep probabilistic framework for modeling the distributions of each time-series together using the soft distributional consistency regularization (\sdcr).
            \model leverages probabilistic deep-learning models to learn priors of individual time-series and refines the priors of all time-series leveraging the hierarchy to provide accurate and well-calibrated forecasts.
      \item \textbf{Adaptation to Strong and Weak Consistency via Soft Distributional Consistency Regularization:} \sdcr imposes soft hierarchical constraints on the full forecast distributions to help adapt the model to varying levels of hierarchical consistency.
            We build a novel refinement module over base forecast priors and leverage multi-task learning over shared parameters that enable \model to perform consistently well across the hierarchy.
      \item \textbf{Evaluation Across Multiple Datasets and with Missing Data:} We show that our method \model outperforms a wide variety of state-of-the-art baselines on both accuracy and calibration, at all levels of the hierarchy, for both strong and weakly consistent datasets.
            We also show training using \sdcr enables \model to leverage hierarchical relations to provide reliable predictions that can handle missing data values in the time-series.
\end{enumerate}

\section{Related work}

\textbf{Probabilistic time-series forecasting} Classical probabilistic
time-series forecasting methods include exponential smoothing and ARIMA
\citep{hyndman2018forecasting}. They are simple but focus on univariate
time-series and model each time-series sequence independently. Recently, deep learning based methods have been successfully applied in this area.
DeepVAR \citep{salinas2020deepar}
trains an auto-regressive recurrent network model on a large number of related time series to directly output the mean and variance parameters of the forecast distribution.
Other works are inspired from the space-state models and explicitly model the transition and emission components
with deep learning modules such as deep Markov models \citep{krishnan2017structured} and  deep state space models \citep{li2021learning, rangapuram2018deep}
Recently, EpiFNP \citep{kamarthi2021doubt} has achieved state-of-art performance in epidemic forecasting.
It learns the stochastic correlations between input data and datapoints to model a flexible non-parametric distribution for univariate sequences.

\noindent \textbf{Hierarchical time-series forecasting} Classical works on hierarchical time-series forecasting used a two-step approach \citep{hyndman2011optimal,hyndman2018forecasting} and focus on point predictions.
They first forecast for time-series only at a single level of the hierarchy
and then derive the forecasts for other nodes using the hierarchical relations.

Recent methods like \mint and \erm are post-processing steps applied on the set of forecasts at all levels of hierarchy.
\mint \citep{wickramasuriya2019optimal, wickramasuriya2021probabilistic} assumes that the base-level forecasts are uncorrelated and unbiased and solve an optimization problem to minimize the variance of forecast errors of past predictions. The unbiased assumption is relaxed in \erm \citep{ben2019regularized}.
\citet{corani2020probabilistic} and \cite{novak2017bayesian} use a fully Bayesian bottom-up post-processing approach using base forecasts from full hierarchy.
Another line of works projects the base forecasts of all time-series into a subspace of consistent forecasts. \citep{erven2015game} use an iterative Game-theoretic approach of minimizing forecast error and projection error.
\citet{taieb2017coherent} uses copula method to refine base forecasts to be distributionally consistent as a post-processing step.
Recent neural methods perform end-to-end learning that enables the model to leverage hierarchical relations while forecasting.
\citet{rangapuram2021end} use a deep-learning based end-to-end approach to directly train on the projected forecasts.
\sharq \citep{han2021simultaneously} is another recent probabilistic deep-learning based  method that uses quantile regression and regularizes for consistency at different quantiles of forecast distribution. However, unlike our approach, these end-to-end methods do not regularize for forecast consistency over the entire distribution (Distributional consistency) but only over fixed quantiles.
Most of these methods also are not designed for cases where the hierarchical constraints are not always consistently followed.

\section{Preliminaries}
\subsection{Problem Statement}
\label{sec:prob}
Consider the dataset $\mathcal{D}$ of $N$ time-series over the time horizon $1,2,\dots,T$. Let $\mathbf{y}_i \in \mathrm{R}^{T}$ be time-series $i$ and $y_{i}^{(t)}$ its value at time $t$. The time-series have a hierarchical relationship denoted as $\mathcal{T} = (G_{\mathcal{T}}, H_{\mathcal{T}})$ where $G_{\mathcal{T}}$ is a tree of $N$ nodes rooted at time-series $1$. For a non-leaf node (time-series) $i$, we denote its children as $\mathcal{C}_i$. The node values are related via set of relations $H_{\mathcal{T}}$ of form
$H_{\mathcal{T}} = \{ \mathbf{y}_{i} = \sum_{j \in \mathcal{C}_i} \phi_{ij}\mathbf{y}_{j}: \forall i \in \{1, 2, \dots, N\}, |\mathcal{C}_i|  > 0 \}$
where values of $\phi_{ij}$ are known and time-independent real-valued constants. 

\begin{definition}[Consistency Error - CE]
    \label{def:ce}
    Given a dataset $\mathcal{D}$ of $N$ time-series over the time horizon $1,2,\dots,T$ and aggregation relations $H_{\mathcal{T}}$ as above, the dataset consistency error (CE) is defined as
        {
            \begin{equation}
                E_{\mathcal{T}}(\mathcal{D}) = \sum_{i\in \{1,2,...N\}, \mathcal{C}_i\neq \emptyset}\left(\mathbf{y}_{i} - \sum_{j \in \mathcal{C}_i} \phi_{ij}\mathbf{y}_{j}\right)^2.
            \end{equation}
        }
    (Intuitively, datasets with lower CE have time-series values which more strictly follow relations $H_\mathcal{T}$).
\end{definition}

\begin{definition}[Strong and weak consistency]
    A dataset $\mathcal{D}$ is strongly consistent if $E_{\mathcal{T}}(\mathcal{D}) = 0$. Otherwise, $\mathcal{D}$ is said to be weakly consistent.
\end{definition}

Let the current time-step be $t$.
For any $1\leq t_1 < t_2\leq t$, we denote $\mathbf{y}_{i}^{(t_1 : t_2)} = \{y_{i}^{(t_1)}, y_{i}^{(t_1 + 1)} , \dots, y_{i}^{(t_2)}\}$.
Given the data $\mathcal{D}^{t} = [\mathbf{y}_{1}^{1 : t}, \mathbf{y}_{2}^{1 : t}, \dots, \mathbf{y}_{N}^{1 : t}]$ and hierarchical relations $H_\mathcal{T}$,
a model $M$ is trained to predict the marginal forecast distributions at time $t+\tau$ for all time-series of hierarchy leveraging past values of all time-series:
$\{p_M(y_1^{(t+\tau)}|\mathcal{D}^t),\dots p_M(y_N^{(t+\tau)}|\mathcal{D}^t)\}$.
Along with the accuracy of probabilistic forecasts, we also evaluate forecast distributions for \emph{calibration}.
We define calibration of model forecasts based on previous works \citep{kamarthi2021doubt,kuleshov2018accurate}:
\begin{definition}(Calibration Score of a Model)
    Given a model $M$ we define a calibration function $k_M: [0,1] \rightarrow [0,1]$ as follows: Given a confidence $c$, $k_M(c)$ is the fraction of the predictions for which the ground truth lies within $c$-confidence interval. The calibration score $CS(M)$ is the total deviation between $c$ and $k_M(c)$: $CS(M) = \int_{0}^1 |k_M(c) - c| dc$. A perfectly calibrated model is such that $\forall c: k_M(c) \approx c$.
\end{definition}

Given a dataset $\mathcal{D}$ with underlying hierarchical relations $H_\mathcal{T}$, the goal of \emph{Calibrated Probabilistic Hierarchical Forecasting} is to design a model $M$ that provides \emph{\underline{accurate}} and \emph{\underline{well-calibrated}} forecast distributions $\{p_M(y_1^{(t+\tau)}|\mathcal{D}^t),\dots p_M(y_N^{(t+\tau)}|\mathcal{D}^t)\}$ across all levels of the hierarchy for both weakly and strongly consistent datasets.

\subsection{Functional Neural Process for Base Forecasts}
\label{sec:ngp}
\model first derives base forecasts for all the node from any differentiable \textit{base forecasting model} such that we can use backpropagation
on the loss function
to update the parameters of the base forecasting model as well.
Formally, the base forecasting model outputs the base forecast distribution parameters $\{\basemu_i, \basesigma_i\}_{i=1}^N$ from input time-series of all nodes
as $P(\{\basemu_i, \basesigma_i\}_{i=1}^N| \{\mathbf{y}_i^{1:t}\}_{i=1}^N)$.

We leverage the recent advances in using Functional Neural Process \cite{louizos2019functional} based non-parametric probabilistic sequential models that have
provided state-of-art accurate and calibrated predictions in many domains \cite{kamarthi2021doubt,kamarthi2021camul}.
These models model the uncertainty of the input time-series as well as its correlation with time-series in training data to provide calibrated forecast distribution.
Specifically, we use a slightly modified version of the model proposed in \cite{kamarthi2021doubt} and denote it as \camul.
The only difference between~\cite{kamarthi2021doubt} and \camul  is that instead of modeling correlations with past values of the same time-series as in univariate time-series forecasting case, we
model correlation with the time-series from all nodes' past.
We provide a detailed description of \camul in the Appendix.
For our further discussion, we can view \camul as a stochastic model with some latent variables:
\begin{equation}
    \begin{split}
        P&(\{\basemu_i, \basesigma_i\}_{i=1}^N| \mathcal{D}^t) =
        \int
        P\left( \mathscr{Z} | \{\mathbf{y}_i^{1:t}\}_{i=1}^N \right)\\
        &\left(\prod_{i=1}^NP(\basemu_i, \basesigma_i | \mathscr{Z})\right) d \mathscr{Z}
    \end{split}
    \label{eqn:process_sum}
\end{equation}
where $\mathscr{Z}$ denotes the full set of latent variables of \camul.
\section{Methodology}
\begin{figure*}[!thb]
    \centering
    \includegraphics[width=.98\linewidth]{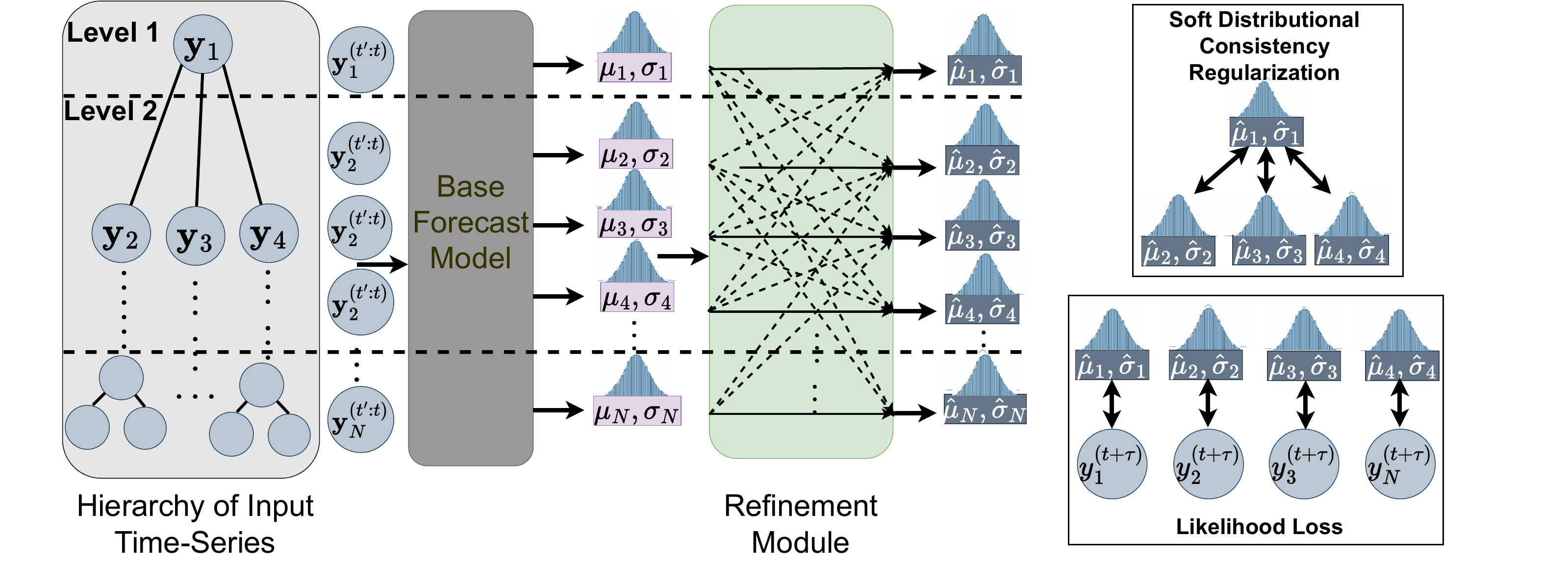}
    \caption{Overview of pipeline of \model.
        The input time series is ingested by \camul, a Functional Neural Process based probabilistic forecasting model, to output the base forecast distribution.
        The parameters of base forecasts are refined by the Hierarchy-aware Refinement module using predictions from all the time-series.
        The training is driven by a likelihood loss that learns from ground truth and Soft Distributional Consistency Regularization that regularizes the forecast distribution to follow the hierarchical relations.
    }
    \label{fig:main}
\end{figure*}
\paragraph{Overview}
\model models the forecast distributions of all time-series nodes of the hierarchy $\{P(y_{i}^{(t+\tau)}|\mathcal{D}^t)\}_{i=1}^N$ by leveraging the relations from the hierarchy to provide accurate and well-calibrated forecasts that are adaptable to varying hierarchical consistency.
Most existing methods do not attempt to model the entire probabilistic distribution but focus on the consistency of point forecasts
or samples or fixed quantiles of the distribution \citep{rangapuram2021end,han2021simultaneously}.
This approach does not fully capture the uncertainty of the forecasts and in turn, does not provide calibrated predictions.
Methods like \pembu, \mint and \erm are post-processing steps that can be applied to base forecasts from any model and provide theoretical guarantees.
However, they do not allow the forecasting model to learn from
relations across the hierarchy.
Moreover, most methods assume that the datasets are strongly consistent over hierarchical relations.
However, many real-world datasets are weakly consistent with time-series values of all nodes of the hierarchy observed simultaneously
and may not follow the hierarchical relations strictly due to noise and discrepancies in collecting data at different levels.
Therefore, most previous works may not adapt well to such deviations from these constraints.

\model, on the other hand, reconciles the need to model consistency between entire forecast distributions as well as induce a soft adaptable constraint to enforce consistency
via a two-step process that is trained in an \textit{end-to-end} manner.
The first component of \model is a differentiable neural probabilistic model such as \camul that
produces a \textit{base forecast} distribution for each node parameterized by $\{(\basemu_i, \basesigma_i)\}_{i=1}^N$.
Base forecasts of all nodes are used as priors to derive a refined set of forecast distributions parameterized by $\{(\refinedmu_i, \refinedsigma_i)\}_{i=1}^N$ via the \textit{Hierarchy-aware Refinement Module} described in Section \ref{sec:refinement}.
We introduce the novel Soft Distributional Consistency Regularization (\sdcr) that enables \model to produce refined forecast distributions
that are distributionally consistent with the hierarchical relation $H_\mathcal{T}$ as described in Section~\ref{sec:loss}.
The full probabilistic process of \model is depicted in Figure \ref{fig:main}.

\subsection{Hierarchy-aware Refinement Module}
\label{sec:refinement}
The base forecast distributions $P(\{\basemu_i, \basesigma_i\}| \mathcal{D}^t)$ produced by \camul (or any other model that can be used in its place)
do not leverage the underlying hierarchical relations $H_{\mathcal{T}}$.
This may lead to sub-optimal forecasting performance and inconsistent forecasts.
The refinement module is a differentiable module that aims to fuse the information from base forecasts of all nodes to output refined forecast distributions that
can leverage \sdcr to be consistent.

Formally, given the parameters of \textit{base forecast} distributions $\{\basemu_i, \basesigma_i\}_{i=1}^N$ derived from \camul for all time-series $\{\mathbf{y}_i^{(t':t)}\}_{i=1}^N$, the refinement module derives
the refined forecast distributions denoted by parameters $\{\refinedmu_i, \refinedsigma_i\}_{i=1}^N$ as functions of parameters of base forecasts of all time-series.
Let $\mathbf{\basemu} = [\basemu_{1} \dots, \basemu_N]$ and $\mathbf{\basesigma} = [\basesigma_{1} \dots, \basesigma_N]$ be vectors of means and standard deviations of base distributions.
Since each of the node's refined distribution parameters depends on all $N$ node's base forecast parameters, the refinement process must be efficient
in fusing the information from all the base forecasts for each node.
Moreover, since we require that \model should be adaptable to datasets of both strong and weak consistency, the refinement process
should automatically learn to trade-off between the influence of base forecast distribution for each node and the fused information from all the nodes.
Considering these objectives,
we derive the mean $\refinedmu_i$ of refined distribution as a weighted sum of two terms: a) $\basemu_i$, the mean of base time-series, and b) linear combination of all base mean of all time-series:
{\begin{equation}
    \gamma_i = \text{sigmoid}(\hat{w}_i), \quad
    \refinedmu_i = \gamma_i \basemu_i + (1-\gamma_i) \mathbf{w}_i^T\mathbf{\basemu}.
    \label{eqn:corem1}
\end{equation}}
$\{\hat{w}_i\}_{i=1}^N$ and $\{\mathbf{w}_{i}\}_{i=1:N}$ are both learnable set of parameters of the model. $\text{sigmoid}(\cdot)$ denotes the sigmoid function.
The operations in Equation \ref{eqn:corem1} have a total computational complexity of $O(N)$ for each node and therefore $O(N^2)$ in total.
This is on par with previous state-of-art end-to-end refinement methods like \hiere \cite{rangapuram2021end} and more efficient than post-processing methods like \mint and \erm during inference.
The learnable parameter $\gamma_i$ allows the refinement module to trade-off between the influence of the base distribution of node $i$ and the influence of the other nodes of the hierarchy making \model automatically adapt to datasets with varying hierarchical consistency.

Similarly, we assume the variance of the refined distribution depends on the base mean and variance of all the time-series.
The variance parameter $\refinedsigma_i$ of the refined distribution is derived from the base distribution parameters $\mathbf{\basemu}$ and $\mathbf{\basesigma}$ as
    {\begin{equation}
            \refinedsigma_i = c\basesigma_i \text{sigmoid}(\mathbf{v}_{1i}^T\mathbf{\basemu} + \mathbf{v}_{2i}^T \mathbf{\basesigma} + b_i)
            \label{eqn:corem2}
        \end{equation}}
where $\{\mathbf{v}_{1i}\}_{i=1}^N$, $\{\mathbf{v}_{2i}\}_{i=1}^N$ and $\{b_i\}_{i=1}^N$ are parameters and $c$ is a positive constant hyperparameter. Note that the complexity of Equation \ref{eqn:corem2} is also $O(N^2)$.

\subsection{Soft Distributional Consistency Regularization}
\label{sec:loss}
While the refinement module helps aggregate information from base forecasts to refine the distribution parameters, we also need to design the loss function
such that parameters of the refinement module and \camul utilize the underlying hierarchical relations $H_{\mathcal{T}}$
to provide hierarchically consistent forecast distributions by effectively utilizing information from all
nodes of the time-series.
For the full distribution of the refined forecasts to be consistent,
we use a \textit{Distributional Consistency Error}(DCE) as part of the loss function and regularize the full distribution of all nodes.

The Distributional Consistency Error (DCE) is defined as follows:
\begin{definition}(Distributional Consistency Error - DCE)
    Given the forecasts at time $t+\tau$ as $\{p_M(y_1^{(t+\tau)}|\mathcal{D}^t),\dots p_M(y_N^{(t+\tau)}|\mathcal{D}^t)\}$ distributional consistency error (DCE) is defined as
        {
            \begin{equation}
                \sum_{i\in \{1,\dots,N\}, \mathcal{C}_i\neq \emptyset} Dist\left(p_M(y_i^{(t+\tau)}|\mathcal{D}^t), p_M(\sum_{j\in \mathcal{C}_i} \phi_{i,j} y_j^{(t+\tau)}|\mathcal{D}^t) \right)
                \label{eqn:dce}
            \end{equation}
        }
    where $Dist$ is a distributional distance metric.
\end{definition}
Leveraging distributional consistency error as a soft regularizer enforces forecast distributions to be well-calibrated
while adaptively adhering to the hierarchical relations of the dataset.

For the distance metric $\text{Dist}$ in Equation \ref{eqn:dce},
we use the Jensen-Shannon Divergence \citep{endres2003new} (JSD) as the distance metric since it is a symmetric
and bounded variant of the popularly used KL-Divergence distance.
Moreover, it assumes a closed form for many widely used distributions including for the Gaussian used in \model.
While we can replace JSD with other distance measures for capturing distributional similarity,
we observed that JSD was sufficient for providing good forecast performance in our applications.
We derive the \textit{distributional consistency error} on Gaussian parameters $\{(\refinedmu_i, \refinedsigma_i)\}_{i=1}^N$ as
    {\begin{equation}
            \begin{split}
                &\mathcal{L}_{\text{dist}} = \sum_{i=1}^N \text{JSD}\left(P(y_i^{(t+\tau)}|\refinedmu_i, \refinedsigma_i), P\left(\sum_{j\in \mathbf{C}_i}\phi_{ij} y_j^{(t+\tau)}| \{\refinedmu_j, \refinedsigma_j\}_{j\in \mathbf{C}_i} \right)\right).
            \end{split}
            \label{eqn:jsd}
        \end{equation}}
The computation of JSD is generally intractable. However, in our case, due to parameterization
of each time-series distribution as a Gaussian we get a closed-form differentiable expression:
{\begin{equation}
    \begin{split}
        \mathcal{L}_{\text{dist}} &= \sum_{i=1}^N \frac{\refinedsigma_i^2 + \left( \refinedmu_i - \sum_{j\in C_i} \phi_{ij}\refinedmu_j \right)^2}{4\sum_{j\in C_i}\phi^2_{ij} \refinedsigma_j^2} +\\
        &\sum_{i=1}^N  \frac{\sum_{j\in C_i}\phi^2_{ij} \refinedsigma_j^2 + \left( \refinedmu_i - \sum_{j\in C_i} \phi_{ij}\refinedmu_j \right)^2}{4\refinedsigma_i^2} -\frac{1}{2}.
    \end{split}
    \label{eqn:jsd2}
\end{equation}}


\paragraph{Derivation of Distributional Consistency Error}
\label{sec:th1}

To derive Equation \ref{eqn:jsd2},  we use the following well-known result for JSD of two Gaussian Distributions \citep{nielsen2019jensen}:
Given two univariate Normal distributions $P_1 = \mathcal{N}_1(\refinedmu_1, \refinedsigma_1)$ and $P_2 = \mathcal{N}_2(\refinedmu_2, \refinedsigma_2)$, the JSD is
\begin{equation}
    \text{JSD}(P_1,P_2) = \frac{1}{2} \left[\frac{\refinedsigma_1^2 + (\refinedmu_1 -\refinedmu_2)^2}{2\refinedsigma_2^2} + \frac{\refinedsigma_2^2 + (\refinedmu_1 -\refinedmu_2)^2}{2\refinedsigma_1^2} - 1\right]
    \label{corr:jsd}
\end{equation}

Consider each $\text{JSD}$ term
of the summation in Equation \ref{eqn:jsd}. Note that
\begin{equation}
    P(y_i^{t+\tau}|\basemu_i, \basesigma_i) = \mathcal{N}(\refinedmu_i,\refinedsigma_i)
    \label{eqn:q1}
\end{equation}
and $P(\sum_{j\in \mathbf{C}_i}\phi_{ij} y_j^{t+\tau}| \{|\basemu_j, \basesigma_j\}_{j\in \mathbf{C}_i} ))$ is weighted sum of Gaussian variables $\{\mathcal{N}(\refinedmu_j, \refinedsigma_j)\}_{j\in C_i}$. Therefore,
\begin{equation}
    P\left(\sum_{j\in \mathbf{C}_i}\phi_{ij} y_j^{t+\tau}| \{\basemu_j, \basesigma_j\}_{j\in \mathbf{C}_i} \right) = \mathcal{N}\left(\sum_{j\in C_i} \phi_{ij} \refinedmu_j, \sqrt{\sum_{j\in C_i} \phi^2_{ij} \refinedsigma_j^2}\right).
    \label{eqn:e2}
\end{equation}
Using Equation \ref{corr:jsd} along with Equations \ref{eqn:q1},\ref{eqn:e2} we get the desired result in Equation \ref{eqn:jsd2}.

We use the distributional consistency error as a soft regularization term to enable  \model to leverage constraints $H_\mathcal{T}$ when generating forecast distributions.
We do not make DCE a hard constraint since the model needs to adapt to datasets of varying consistency.
Particularly, for weakly consistent datasets, we do not require \model to strictly adhere the hierarchical relations $H_\mathcal{T}$ which may result in
sub-optimal forecast accuracy and calibration, since the ground truth does not follow $H_\mathcal{T}$ as well.
Therefore, using DCE as a soft-regularizer allows the model to adapt to varying strictness of $H_\mathcal{T}$ across different domains.

\subsection{Details on Training}
\label{sec:train_steps}

\paragraph{Training loss}
Along with the \sdcr loss $\mathcal{L}_{\text{dist}}$ which informs \model of hierarchical relations $H_\mathcal{T}$ optimizes for distributional consistency,
we derive the \textit{Likelihood Loss} $\mathcal{L}_{\text{ll}}$ to optimize for the accuracy and calibration of the forecasts.
Using \camul (or any other base forecasting model with latent variables) as the, the full probabilistic process of \model can be summarized as:
\begin{equation}
    \begin{split}
        &P(\{y_i^{(t+\tau)}\}_{i=1}^N| \mathcal{D}^t) =\\
        \int
        &\underbrace{P\left( \mathscr{Z} | \{\mathbf{y}_i^{1:t}\}_{i=1}^N \right)
            \left(\prod_{i=1}^NP(\basemu_i, \basesigma_i | \mathscr{Z})\right)}_{\text{\camul (Base forecasts)}} \\
        &\underbrace{\prod_{i=1}^N P(\refinedmu_i, \refinedsigma_i| \{\basemu_j, \basesigma_j\}_{j=1}^N) P(y_i^{(t+\tau)}|\refinedmu_i, \refinedsigma_i)}_{\text{Refinement}} d\mathscr{Z}.
    \end{split}
    \label{eqn:process}
\end{equation}
Integrating over the latent variables $\mathscr{Z}$ in Equation \ref{eqn:process} is highly intractable.
Therefore, we use variational inference by approximating the posterior over the latent variables $P(\mathscr{Z}| \{y_i^{(t+\tau)}\}_{i=1}^N)$
and derive an ELBO $\mathcal{L}_{\text{ll}}$ which we use as the optimization objective.
The details of the derivation of ELBO loss are in the Appendix.
Since the refinement module is a deterministic mapping from base to refined distribution parameters, the ELBO derivation is very similar to that in~\cite{kamarthi2021doubt}.
Therefore, our framework is flexible to adapt to a wide range of neural forecasting models with different learning algorithms.

Thus, the total loss for training is given as
$\mathcal{L} = \mathcal{L}_{\text{ll}} + \lambda \mathcal{L}_{\text{dist}}$
where the hyperparameter $\lambda$ controls the trade-off in importance between data likelihood and distributional consistency.
We also use the reparameterization trick to make the sampling process differentiable and we learn the parameters of all training modules via Stochastic Variational Bayes \citep{kingma2013auto}. The full pipeline of \model is summarized in Figure \ref{fig:main}.

\paragraph{Parameter sharing across nodes}
Since \model's \camul module forecasts for multiple nodes, we leverage the hard-parameter sharing paradigm of multi-task learning \citep{caruana1997multitask} and use a different set of parameters for Predictive Distribution Decoder (i.e., weights of $\Theta_3$) whereas the parameters of other components of \camul are shared across all nodes (Figure \ref{fig:main}). Sharing parameters for Probabilistic Neural Encoder drastically lowers the number of learnable parameters since datasets can have a large number of nodes (up to 512 nodes in our experiments).

\paragraph{Pre-training on individual time-series}
Before we start training for refined forecasts, we pre-train the parameters of \camul on given training dataset to model base forecast distribution accurately.
We pre-train using only log-likelihood loss to learn parameters $\{\basemu_i, \basesigma_i\}_{i=1}^N$.

\section{Experiments}
\label{sec:expt}
We evaluate \model over multiple datasets and compare it with state-of-the-art baselines\footnote{Code and datasets: \url{https://github.com/AdityaLab/Profhit}}\!.

\subsection{Setup}
\paragraph{Baselines:} We compare \model's performance against state-of-the-art hierarchical forecasting methods. We also compare against state-of-the-art general probabilistic forecasting methods to study the importance of modeling the hierarchy for both weak and strongly consistent datasets.
\begin{enumerate}
    \item \textbf{\camul} \citep{kamarthi2021doubt}: a neural forecasting model for accurate and calibrated forecasts described in Section \ref{sec:ngp}
    \item \textbf{\deepvar} \citep{salinas2020deepar}: another state-of-the-art deep probabilistic forecasting models which do not exploit hierarchy relations.
    \item \textbf{\mint} \citep{wickramasuriya2019optimal}: a post-processing method for reconciliation of base forecasts
    \item \textbf{\erm} \citep{ben2019regularized}: another post-processing method like \mint that relaxes unbiased assumptions of base forecasts
    \item \textbf{\hiere} \citep{rangapuram2021end} is a recent state-of-the-art deep learning based approach that projects the base predictions onto a space of consistent forecasts and trains the model in an end-to-end manner.
    \item \textbf{\sharq} \citep{han2021simultaneously} is another state-of-the-art deep learning based approach that reconciles forecast distributions by using quantile regressions and making the quantile values consistent.
    \item \textbf{\pembu} \citep{taieb2017coherent} is a post-processing method that refines base forecasts to be distributionally consistent.
\end{enumerate}
Note: In our experiments, ee performed \erm and \mint on Monte Carlo samples of \camul predictive distribution since \camul provided better results compared to \deepvar.
We also use the mean forecast from \mint and \erm as input forecasts for \pembu.

\hide{
    We also evaluate the efficacy and contribution of our various modeling choices by performing an ablation study using the following variants of \model:  (7) \textbf{\pglobal:} We study the effect of our multi-tasking hard-parameter sharing approach (Section \ref{sec:train_steps}) by training a variant where all the parameters are shared across all the nodes.
    (8) \textbf{\finetune:} We also look at the efficacy of our soft regularization using both losses that adapts to optimize for both consistency and training accuracy by comparing it with a variant where the predictive distribution decoder parameters are further fine-tuned for individual nodes using only the likelihood loss.
    (9) \textbf{\pvar:} We evaluate our choice of using \camul, a previous state-of-the-art univariate forecasting model for accurate and calibrated forecasts with DeepAR, another popular probabilistic forecasting model that was used by \hiere.
    (10) \textbf{\nocoherent:} This variant is trained by completely removing the \sdcr from the training. Note that unlike \finetune which was initially trained with \sdcr before fine-tuning,  \nocoherent never uses the \sdcr at any point of the training routine. Therefore \nocoherent measures the importance of explicitly regularizing over the information from the hierarchy.

}

\begin{table}[h]
    \caption{Dataset Characteristics and Consistency}
    \centering
    \scalebox{0.85}{
        \begin{tabular}{c|ccccc}
            Dataset  & \multicolumn{1}{l}{No. of Nodes} & \multicolumn{1}{l}{\begin{tabular}[c]{@{}l@{}}Levels of \\ Hierarchy\end{tabular}} & \multicolumn{1}{l}{\begin{tabular}[c]{@{}l@{}}$\tau$\end{tabular}} & \multicolumn{1}{l}{\begin{tabular}[c]{@{}c@{}}Obs.\\ per node\end{tabular}} & \multicolumn{1}{l}{\begin{tabular}[c]{@{}c@{}} Consistency\\(CE)\end{tabular}} \\ \hline
            \tourism & 555                              & \multicolumn{1}{c}{4,5}                                                            & 12                                                                 & 228                                                                         & Strong(0)                                                                      \\
            \labour  & 57                               & 4                                                                                  & 8                                                                  & 514                                                                         & Strong(0)                                                                      \\
            \wiki    & 207                              & 5                                                                                  & 1                                                                  & 366                                                                         & Strong(0)                                                                      \\
            \symp    & 61                               & 3                                                                                  & 4                                                                  & 544                                                                         & Weak(3.37)                                                                     \\
            \fbsymp  & 61                               & 3                                                                                  & 4                                                                  & 257                                                                         & Weak(2.44)
        \end{tabular}
    }
    \label{tab:dataset}
\end{table}
\paragraph{Datasets:}
We evaluate on a diverse set of publicly available datasets (Table \ref{tab:dataset}) from different domains with varied hierarchical relations and consistency.
The benchmarking dataset and evaluation setup is replicated from recent and past literature related to general hierarchical forecasting as well as epidemic forecasting.
\begin{enumerate}
    \item \labour dataset contains monthly employment data from Feb 1978 to Dec 2020 collected from the Australian Bureau of Statistics.
    \item \tourism \citep{wickramasuriya2019optimal} contains tourism flows in different regions in Australia grouped via region and demographic. It has two sets of hierarchies (with four and five levels), one for the mode of travel and the other for geography with the top node being the only common node of both hierarchies.
    \item \wiki dataset collects the number of daily views of 145000 Wikipedia articles aggregated into 150 groups \citep{taieb2017coherent}. These 150 groups are leaf nodes of a four-level hierarchy with groups of similar topics aggregated together.
    \item \symp contains flu incidence values called \textit{weighted influenza-like incidence} (wILI) values \citep{reich2019collaborative} at multiple spatial scales for USA for period of 2004-2020. The scales used are states, HHS and National level (US states are grouped into 10 HHS regions by CDC).
    \item \fbsymp provides an aggregated anonymized daily indicator for the prevalence of Covid-19 symptoms based on online surveys conducted on Facebook \citep{delphifb2022} from Dec 2020 to Aug 2021 for each state and national level. We use the state-level values to find aggregates at HHS levels. 
\end{enumerate}

\tourism, \labour and \wiki are constructed by collecting values
of leaf nodes and deriving the values of the time series of other nodes of the hierarchy. Hence, they are strongly consistent with zero CE (Definition \ref{def:ce}). The values of each node of the hierarchy in the case of \symp and \fbsymp are directly collected or measured. For example,
the values of Flu-Symptoms dataset are collected from public health agencies at the state, HHS
and national levels and aggregated by CDC. Due to factors like reporting discrepancies and noise,
they contain values in time series that may deviate from the given hierarchical relations \citep{chakraborty2018know}.
Therefore, these datasets are weakly consistent with significant CE (Table \ref{tab:dataset}).
We also provide level wise consistency errors for all the datasets in the Appendix.

\paragraph{Evaluation metrics}
We evaluate our model and baselines using carefully chosen metrics that are widely used in the literature to measure accuracy and calibration.
We also measure the distributional consistency of the output forecast to study how well the model trade-off accuracy and consistent for datasets of varying consistency errors.
For a ground truth $y^{(t)}$, let the predicted probability distribution be $\hat{p}_{y^{(t)}}$ with mean $\hat{y}^{(t)}$. Also let $\hat{F}_{y^{(t)}}$ be the CDF.
    \noindent$\bullet$\textbf{Mean Absolute Percentage Error (MAPE)} is a commonly used score for point-predictions calculated as
    \[MAPE = \frac{1}{N}\sum_{t=t_1}^{t_N} |\frac{y^{(t)} - \hat{y}^{(t)}}{y^{(t)}}|\]
    
    \noindent$\bullet$\textbf{Log Score (LS)} is a standard score used to measure the accuracy of probabilistic forecasts in epidemiology \citep{reich2019collaborative}. LS measures the negative log likelihood of a fixed size interval around the ground truth under the predictive distribution:
    \[LS(\hat{p}_y,y) = - \int_{y-L}^{y+L} \log \hat{p}_y(\hat{y}) d\hat{y}.\]
    Similar to \citep{reich2019collaborative}, the log-likelihood of a forecast is capped at -10.
    The calculation of LS is tractable due to the gaussian assumption on the forecast distribution.
    
    \noindent$\bullet$\textbf{Calibration Score (CS):} To measure the calibration of forecasts, we use the calibration score defined in Section \ref{sec:prob}.
    We approximate the integral via Riemann sum over $[0,1]$ with step-size 0.05.
    
    \noindent$\bullet$\textbf{Cumulative Ranked Probability Score (CRPS)} is a widely used standard metric for the evaluation of probabilistic forecasts that measures \textit{both accuracy and calibration}.
    Given ground truth $y$ and the predicted probability distribution $\hat{p}_y$, let $\hat{F}_y$ be the CDF. Then, CRPS is defined as:
    \[CRPS(\hat{F}_y, y) = \int_{-\infty}^\infty (\hat{F}_y(\hat{y}) - \mathbf{1}\{\hat{y}>y\})^2 d\hat{y}.\] We approximate $\hat{F}_y$ as a Gaussian distribution formed from samples of the model to derive CRPS.
    
    \noindent$\bullet$\textbf{Distributional Consistency Error (DCE):} We calculate the Distributional Consistency Error (Equation \ref{eqn:jsd}) on output forecast distributions during inference to study how \model and baselines leverage \sdcr to learn from hierarchical relations across datasets of varying consistency and trade-off consistent, calibration and accuracy, especially for weakly consistent data (Section \ref{sec:results} Q3).


\subsection{Results}
\label{sec:results}
We comprehensively evaluate  \model  through the following questions:
\noindent \textbf{Q1}: Does \model predict accurate calibrated forecasts?
\noindent \textbf{Q2}: Does \model provide consistently better performance across all levels of the hierarchy?
\noindent \textbf{Q3}: Does \sdcr help \model outperform baselines on both strongly and weakly consistent datasets?
\noindent \textbf{Q4}: What impact do various modeling choices have on the model's overall performance?
\noindent \textbf{Q5}: How does improved calibration and forecast consistency help \model
deal with missing values in data?

\paragraph{Accuracy and calibration performance (Q1)}
\begin{table*}[h]
    \caption{Average scores (across 5 runs)  across all levels of hierarchy for all baselines, \model and its variants.
        \model provides top performance in terms of all evaluation metrics in most of the benchmarks.
    }
    \label{tab:tab1}
    \scalebox{0.95}{
        \begin{tabular}{@{}c@{}}
            \begin{minipage}{\textwidth}
                \centering
                \begin{tabular}{c|ccccc|ccccc|ccccc}
                                         & \multicolumn{5}{c|}{\tourism} & \multicolumn{5}{c|}{\labour} & \multicolumn{5}{c|}{\wiki}                                                                                                                                                                                                       \\ \hline
                    \textbf{Models/Data} & \textbf{MAPE}\%               & \textbf{CRPS}                & \textbf{LS}                & \textbf{CS}   & \textbf{DCE}  & \textbf{MAPE}\% & \textbf{CRPS}  & \textbf{LS}   & \textbf{CS}   & \textbf{DCE}  & \textbf{MAPE}\% & \textbf{CRPS}  & \textbf{LS}   & \textbf{CS}   & \textbf{DCE}  \\ \hline
                    \deepvar             & 3.12                          & 0.17                         & 0.61                       & 0.19          & 0.32          & 18.27           & 0.045          & 0.75          & 0.25          & 0.34          & 16.52           & 0.232          & 0.83          & 0.27          & 0.26          \\
                    \camul               & 2.28                          & 0.21                         & 1.19                       & 0.14          & 0.39          & 14.52           & 0.071          & 1.41          & 0.21          & 0.22          & 15.63           & 0.287          & 0.86          & 0.21          & 0.39          \\
                    TSFNP-MinT           & 1.17                          & 0.5                          & 0.58                       & 0.15          & 0.24          & 16.46           & 0.045          & 4.12          & 0.26          & 0.12          & 13.79           & 0.243          & 0.78          & 0.18          & 0.18          \\
                    TSFNP-ERM            & \textbf{1.42}                 & 0.56                         & 0.53                       & 0.11          & 0.18          & 13.57           & 0.045          & 3.63          & 0.23          & 0.19          & 17.74           & 0.221          & 0.74          & 0.19          & 0.21          \\
                    \hiere               & 1.67                          & 0.15                         & 0.38                       & 0.17          & 0.21          & \textbf{12.53}  & 0.034          & 0.51          & 0.25          & 0.15          & 17.05           & 0.211          & 0.46          & 0.23          & 0.12          \\
                    SHARQ                & 1.63                          & 0.17                         & 0.41                       & 0.12          & 0.13          & 14.21           & 0.054          & 0.47          & 0.18          & 0.09          & 16.13           & 0.241          & 0.52          & 0.16          & 0.16          \\
                    PEMBU-\mint          & 1.77                          & 0.15                         & 0.46                       & 0.24          & 0.03          & 13.55           & 0.039          & 0.56          & 0.22          & 0.11          & 14.66           & 0.279          & 0.58          & 0.21          & 0.05          \\
                    PEMBU-ERM            & 1.63                          & 0.16                         & 0.43                       & 0.21          & \textbf{0.02} & 13.19           & 0.042          & 0.61          & 0.25          & \textbf{0.03} & 15.79           & 0.268          & 0.54          & 0.18          & \textbf{0.02} \\ \hline
                    \model               & 1.47                          & \textbf{0.12}                & \textbf{0.33}              & \textbf{0.09} & \textbf{0.02} & 12.79           & \textbf{0.026} & \textbf{0.21} & \textbf{0.14} & 0.05          & \textbf{12.47}  & \textbf{0.184} & \textbf{0.35} & \textbf{0.13} & 0.04          \\
                \end{tabular}
            \end{minipage}
            \\
            \begin{minipage}{\textwidth}
                \centering
                \begin{tabular}{c|ccccc|ccccc}
                                         & \multicolumn{5}{c|}{\symp} & \multicolumn{5}{c|}{\fbsymp}                                                                                                                                    \\ \hline
                    \textbf{Models/Data} & \textbf{MAPE}\%            & \textbf{CRPS}                & \textbf{LS}   & \textbf{CS}    & \textbf{DCE}  & \textbf{MAPE}\% & \textbf{CRPS} & \textbf{LS}   & \textbf{CS}   & \textbf{DCE}  \\ \hline
                    \deepvar             & 31.27                      & 0.610                        & 3.25          & 0.065          & 0.31          & 17.39           & 7.32          & 5.32          & 0.17          & 0.29          \\
                    \camul               & 12.8                       & 0.460                        & 0.93          & 0.034          & 0.42          & 15.35           & 5.53          & 7.84          & 0.11          & 0.37          \\
                    TSFNP-\mint          & 10.56                      & 0.630                        & 3.18          & 0.082          & 0.18          & 12.24           & 5.39          & 6.35          & 0.14          & 0.24          \\
                    TSFNP-\erm           & 11.85                      & 0.620                        & 2.75          & 0.075          & 0.12          & 13.16           & 6.14          & 4.23          & 0.12          & 0.19          \\
                    \hiere               & 15.67                      & 0.420                        & 0.81          & 0.12           & 0.32          & 12.63           & 4.12          & 1.13          & 0.19          & 0.26          \\
                    SHARQ                & 18.34                      & 0.470                        & 1.42          & 0.071          & 0.21          & 12.82           & 3.12          & 0.81          & 0.15          & 0.19          \\
                    PEMBU-MinT           & 15.44                      & 0.621                        & 2.55          & 0.18           & \textbf{0.05} & 13.75           & 5.78          & 4.22          & 0.22          & \textbf{0.07} \\
                    PEMBU-ERM            & 17.57                      & 0.688                        & 2.74          & 0.15           & 0.07          & 12.99           & 6.31          & 5.18          & 0.18          & 0.1           \\ \hline
                    \model               & \textbf{8.85}              & \textbf{0.250}               & \textbf{0.28} & \textbf{0.042} & 0.14          & 9.67            & \textbf{1.43} & \textbf{0.45} & \textbf{0.08} & 0.16          \\
                \end{tabular}
            \end{minipage}
        \end{tabular}
    }
\end{table*}

We evaluate all baselines, \model and its variants for all the datasets over 5 independent runs.
The average scores across all levels hierarchy are shown in Tables \ref{tab:tab1}.
\model significantly outperforms all baselines in MAPE score by 13\%.
It also outperforms the baselines in LS and CS significantly in most cases.
Finally, \model shows 41-88\% better CRPS scores.
Thus, \model adapts well to varied kinds of datasets and outperforms all baselines in both accuracy and calibration.

\paragraph{Performance across the hierarchy (Q2)}
Next, we look at the performance of all models across each level of the hierarchy.
We compared the performance of \model with best-performing baselines \hiere and \sharq for all datasets.
\model significantly outperforms the best baselines both in terms of accuracy and calibration.
The performance improvement is consistent across all levels of the hierarchy in most of the benchmarks.
We show detailed results in the Appendix.
This shows that the model effectively leverages hierarchical relations across all nodes to
provide significantly more accurate and calibrated forecasts across the hierarchy.

\paragraph{Effect of \sdcr on datasets of varying consistency (Q3)}
Since most previous state-of-the-art models assume datasets to be strongly consistent, deviations
from this assumption can cause under-performance when used with weakly consistent datasets.
This is evidenced in Table \ref{tab:tab1} where most of the baselines
explicitly optimize for hierarchical consistency as a hard constraint on the forecasts. For example, \pembu's forecasts have better distributional consistency error (DCE) for weakly consistent datasets. However, they perform much
worse in both accuracy and calibration than even \camul, which does not even leverage hierarchical relations. Since we use \sdcr as a soft learning constraint,
\model can learn to trade-off consistency for accuracy and calibration. Therefore, \model provides 93\% better CRPS and significantly better calibration over the best baselines.
These improvements are more pronounced at non-leaf nodes of hierarchy where \model's performance is significanly larger for the
weakly consistent \symp and \fbsymp datasets.
In the case of strongly consistent datasets, \model
provides 54\% better CRPS and better calibration while having comparable DCE to \pembu.
We provide further analysis of these observations in the Appendix.

\paragraph{Ablation study on various modeling choices (Q4)}

We evaluate the efficacy and contribution of our various modeling choices including the usefulness of \sdcr, refinement module, hard-parameter sharing, and using \camul as
our model of choice for base forecasts.
We perform an ablation study using the following variants of \model:

    \noindent$\bullet$ \textbf{\nocoherent:} This variant is trained by completely removing the \sdcr from the training. Note that unlike \finetune which was initially trained with \sdcr before fine-tuning,  \nocoherent never uses the \sdcr at any point of the training routine. Therefore \nocoherent measures the importance of explicitly regularizing over the information from the hierarchy.
    
    \noindent$\bullet$ \textbf{\norefine}: We remove the hierarchical refinement module and optimize the base forecasts for both likelihood and \sdcr loss.
    
    \noindent$\bullet$ \textbf{\pvar:} We evaluate our choice of using \camul, a previous state-of-the-art univariate forecasting model for accurate and calibrated forecasts by replacing it with DeepAR\cite{salinas2020deepar}, another popular probabilistic forecasting model that was used by \hiere.
    
    \noindent$\bullet$ \textbf{\pglobal:} We study the effect of our multi-tasking hard-parameter sharing approach (Section \ref{sec:train_steps}) by training a variant where all the parameters for each \camul forecasting module are shared across all the nodes.
    
    \noindent$\bullet$ \textbf{\finetune:} We also look at the efficacy of our soft regularization using both losses that adapt to optimize for both consistency and training accuracy by comparing it with a variant where the predictive distribution decoder parameters are further fine-tuned for individual nodes using only the likelihood loss.

\begin{table*}[h]
    \caption{Average scores (across 5 runs)  across all levels of hierarchy for \model and its ablation variants.
        The best score is bolded and the second best is \underline{underlined}.
    }
    \label{tab:tab_abl}
    \scalebox{0.95}{

        \centering
        \begin{tabular}{c|ccc|ccc|ccc|ccc|ccc}
                                 & \multicolumn{3}{c}{\tourism} & \multicolumn{3}{c}{\labour} & \multicolumn{3}{c}{\wiki} & \multicolumn{3}{c}{\symp} & \multicolumn{3}{c}{\fbsymp}                                                                                                                                                                                                  \\ \hline
            \textbf{Models/Data} & \textbf{CRPS}                & \textbf{CS}                 & \textbf{DCE}              & \textbf{CRPS}             & \textbf{CS}                 & \textbf{DCE}     & \textbf{CRPS}     & \textbf{CS}      & \textbf{DCE}     & \textbf{CRPS}     & \textbf{CS}       & \textbf{DCE}     & \textbf{CRPS}    & \textbf{CS}      & \textbf{DCE}     \\ \hline
            \model               & \textbf{0.12}                & \underline{0.09}            & \underline{0.02}          & \textbf{0.026}            & \textbf{0.14}               & \underline{0.05} & \textbf{0.184}    & \textbf{0.13}    & \textbf{0.04}    & 0.250             & \underline{0.042} & \underline{0.14} & 1.43             & \underline{0.08} & 0.16             \\ \hline
            \nocoherent          & 0.18                         & 0.21                        & 0.35                      & 0.043                     & 0.26                        & 0.17             & 0.227             & 0.35             & 0.14             & \underline{0.248} & 0.16              & 0.22             & \underline{1.17} & 0.24             & 0.22             \\
            \norefine            & 0.16                         & 0.14                        & 0.19                      & 0.037                     & 0.18                        & 0.15             & 0.219             & 0.19             & 0.09             & 0.256             & 0.097             & 0.17             & \textbf{1.15}    & 0.12             & 0.18             \\
            \pvar                & \underline{0.13}             & 0.12                        & 0.04                      & 0.029                     & 0.17                        & 0.08             & 0.201             & 0.24             & \underline{0.07} & 0.361             & 0.083             & 0.15             & 2.13             & 0.18             & \underline{0.15} \\
            \finetune            & 0.16                         & 0.14                        & 0.25                      & 0.031                     & 0.21                        & 0.13             & 0.216             & 0.21             & 0.08             & \textbf{0.240}    & \textbf{0.039}    & 0.17             & 1.18             & \textbf{0.07}    & 0.19             \\
            \pglobal             & \underline{0.13}             & \textbf{0.06}               & \textbf{0.01}             & \underline{0.027}         & \underline{0.16}            & \textbf{0.04}    & \underline{0.185} & \underline{0.16} & \textbf{0.04}    & 0.350             & 0.086             & \textbf{0.09}    & 2.64             & 0.14             & \textbf{0.11}
        \end{tabular}
    }
\end{table*}

We compare the performance of \model with its variants in CRPS, CS and DCE in table \ref{tab:tab_abl} (Rest of the metrics are in Appendix).
We observe that \model is comparable to or better than the best-performing variant in most cases.
We observe that the best-performing variant for strongly consistent datasets is \pglobal which is trained with both likelihood loss and \sdcr (Table \ref{tab:tab1}).
But its performance severely degrades for weakly consistent datasets since sharing all model parameters across all time-series makes it inflexible to model patterns and deviations specific to individual nodes.
In contrast, \finetune and \nocoherent performs the best among variants for weakly consistent datasets since they train separate sets of decoder parameters for each node. But they perform poorly for strongly consistent datasets since they don't leverage Distributional Consistency effectively.
\model combines the flexible parameter learning of \finetune and leverage Distributional Consistency to jointly optimize the parameters like \pglobal providing comparable performance to best variants over all datasets.

\paragraph{Adapting to missing data (Q5)}

\begin{figure*}[ht]
  \centering
  \vspace{-0.15in}
  \begin{subfigure}[b]{0.195\linewidth}
    \centering
    \includegraphics[width=.9\linewidth]{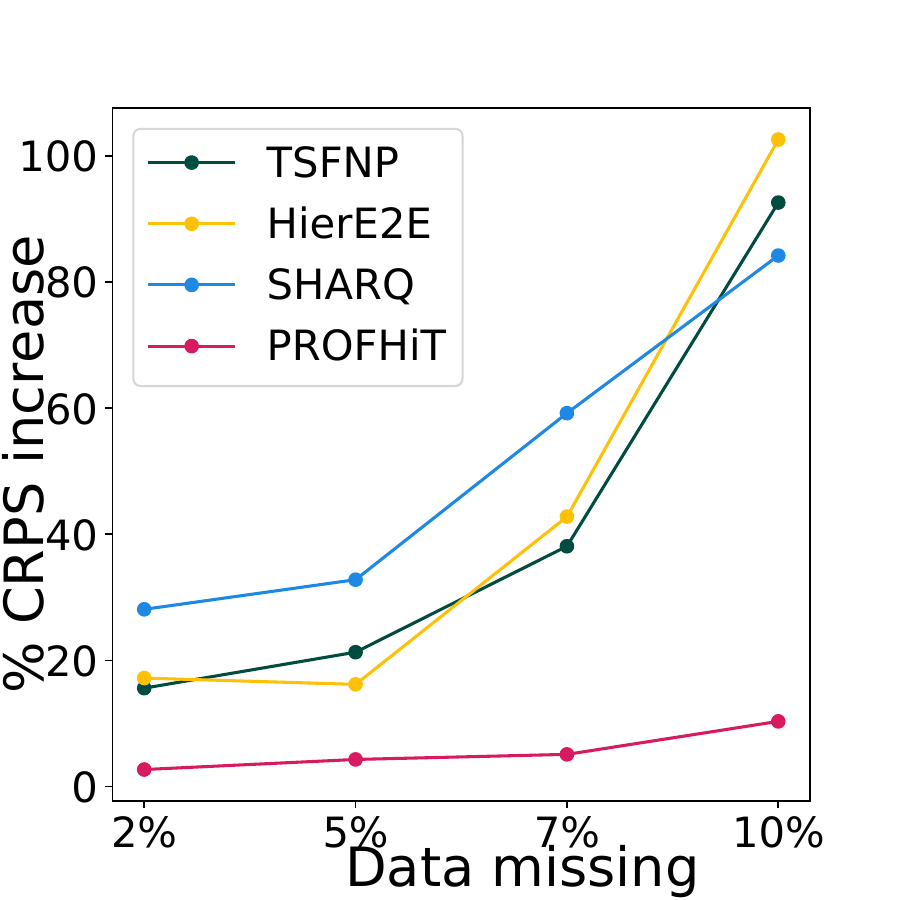}
    \caption{\tourism}
  \end{subfigure}%
  \begin{subfigure}[b]{0.195\linewidth}
    \centering
    \includegraphics[width=.9\linewidth]{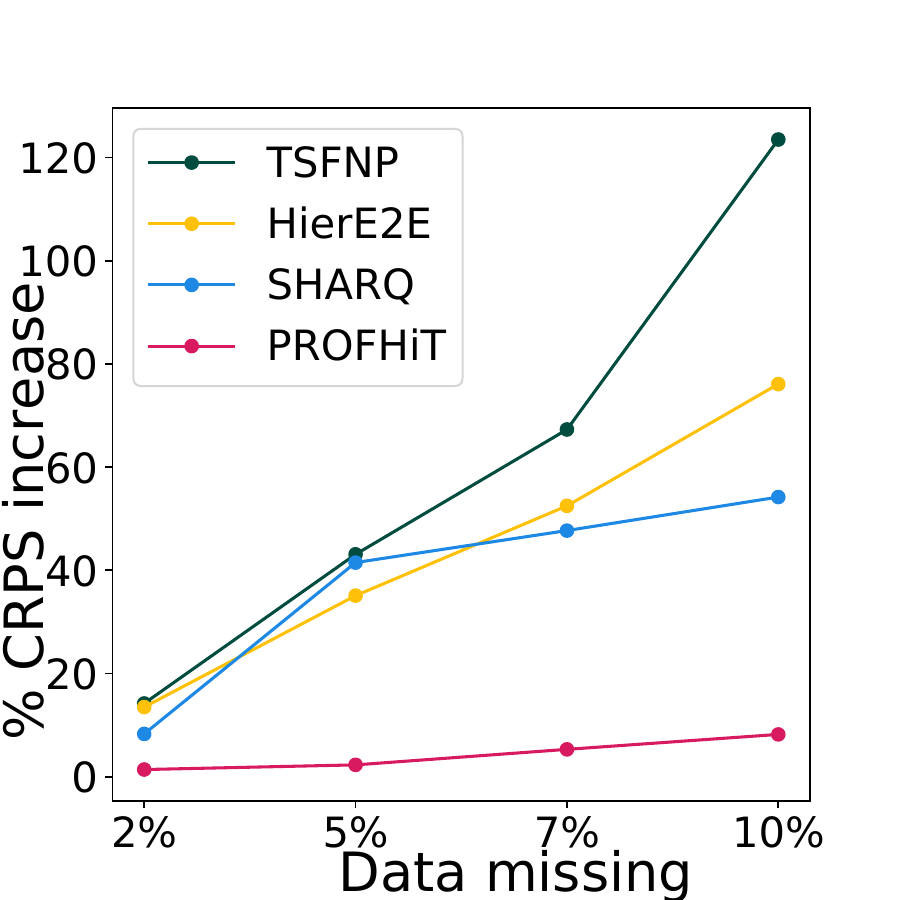}
    \caption{\labour}
  \end{subfigure}%
  \begin{subfigure}[b]{0.195\linewidth}
    \centering
    \includegraphics[width=.9\linewidth]{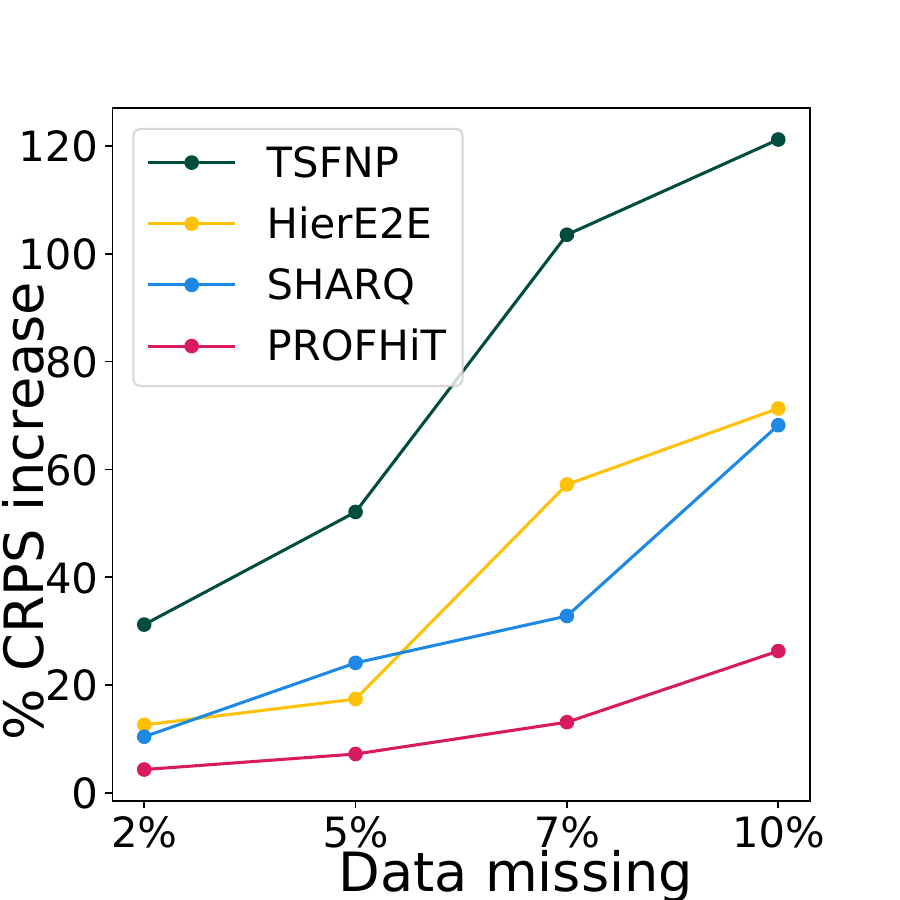}
    \caption{\wiki}
  \end{subfigure}%
  \begin{subfigure}[b]{0.195\linewidth}
    \centering
    \includegraphics[width=.9\linewidth]{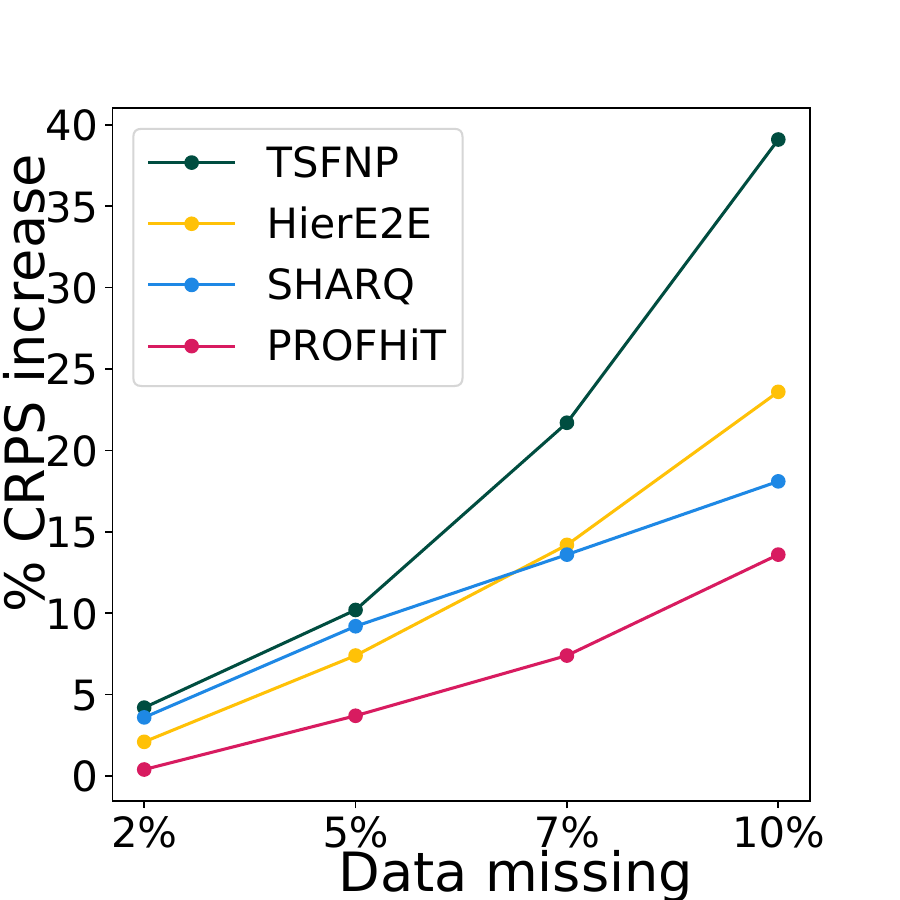}
    \caption{\symp}
  \end{subfigure}%
  \begin{subfigure}[b]{0.195\linewidth}
    \centering
    \includegraphics[width=.9\linewidth]{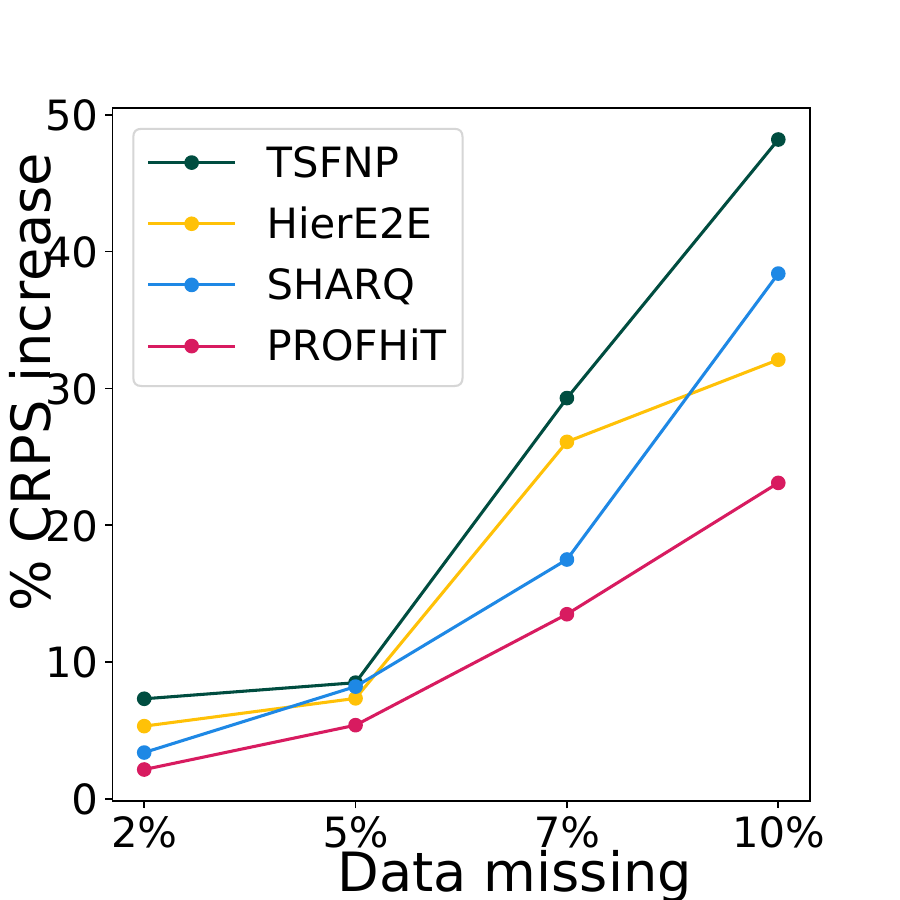}
    \caption{\fbsymp}
  \end{subfigure}
  \vspace{-0.05in}
  \caption{\% increase in CRPS for all models with increase in proportion of missing data.}
  \vspace{-0.05in}
  \label{fig:miss1}
\end{figure*}
Accurate and well-calibrated models that can effectively leverage the knowledge of the hierarchy can intuitively allow models to better adapt to noise/missing data.
Hence, we introduce the task of \textit{Hierarchical Forecasting with Missing Values} and
study the adaptation of models when there are missing values in time-series. We model a  situation that is encountered in many real-world applications such as Epidemic Forecasting where the past few values of time-series are missing due to various factors like data reporting delays \citep{chakraborty2018know}.

Formally, at time-period $t$, we are given full data up to time $t-\rho$. We set $\rho=5$ since it is the average forecast horizon of all datasets.
For sequence values in the period between $t-\rho$ and $t$, we randomly remove $k\%$ of these values across all time-series.
The models are trained on the complete time-series dataset till time $t'=t-\rho$. Models' predictions are then used to fill in missing values for time $t'$ to $t$.
Finally, we input the filled time-series to generate the forecasts for future time-steps.

We measure the relative decrease in performance of \model and baselines with an increase in the percentage of missing data $k$ (Figures \ref{fig:miss1}).
We observe that \model's performance decrease as the fraction of missing values increases is much slower compared to other baselines. Even at $k=10\%$, \model's performance decreases by 10.45-26.8\% compared to other baselines that typically decrease by over 70\%.
Thus, \model effectively uses hierarchical relations to generate reliable predictions on strong and weakly consistent datasets.

\begin{figure*}[ht]
  \centering
  \hide{
    \begin{subfigure}[b]{0.195\linewidth}
      \centering
      \includegraphics[width=.9\linewidth]{Images/missing_Tourism.pdf}
      \caption{\tourism}
    \end{subfigure}%
    \begin{subfigure}[b]{0.195\linewidth}
      \centering
      \includegraphics[width=.9\linewidth]{Images/missing_Labour.pdf}
      \caption{\labour}
    \end{subfigure}%
    \begin{subfigure}[b]{0.195\linewidth}
      \centering
      \includegraphics[width=.9\linewidth]{Images/missing_Wiki.pdf}
      \caption{\wiki}
    \end{subfigure}%
    \begin{subfigure}[b]{0.195\linewidth}
      \centering
      \includegraphics[width=.9\linewidth]{Images/missing_Flu.pdf}
      \caption{\symp}
    \end{subfigure}%
    \begin{subfigure}[b]{0.195\linewidth}
      \centering
      \includegraphics[width=.9\linewidth]{Images/missing_FBS.pdf}
      \caption{\fbsymp}
    \end{subfigure}}
  \begin{subfigure}[b]{0.195\linewidth}
    \centering
    \includegraphics[width=.9\linewidth]{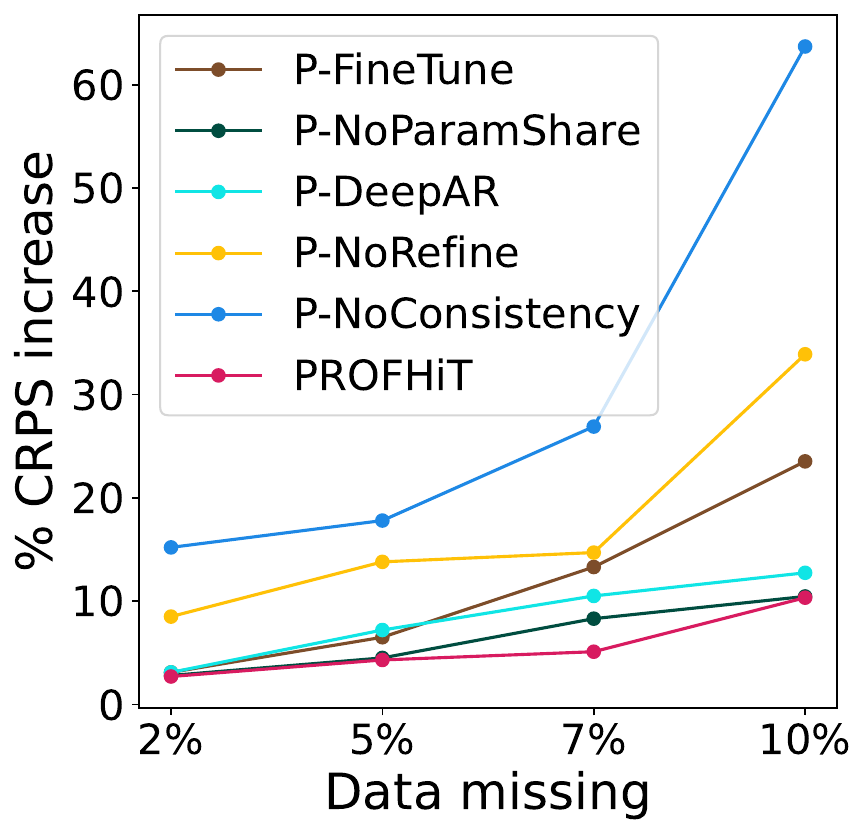}
    \caption{\tourism}
  \end{subfigure}%
  \begin{subfigure}[b]{0.195\linewidth}
    \centering
    \includegraphics[width=.9\linewidth]{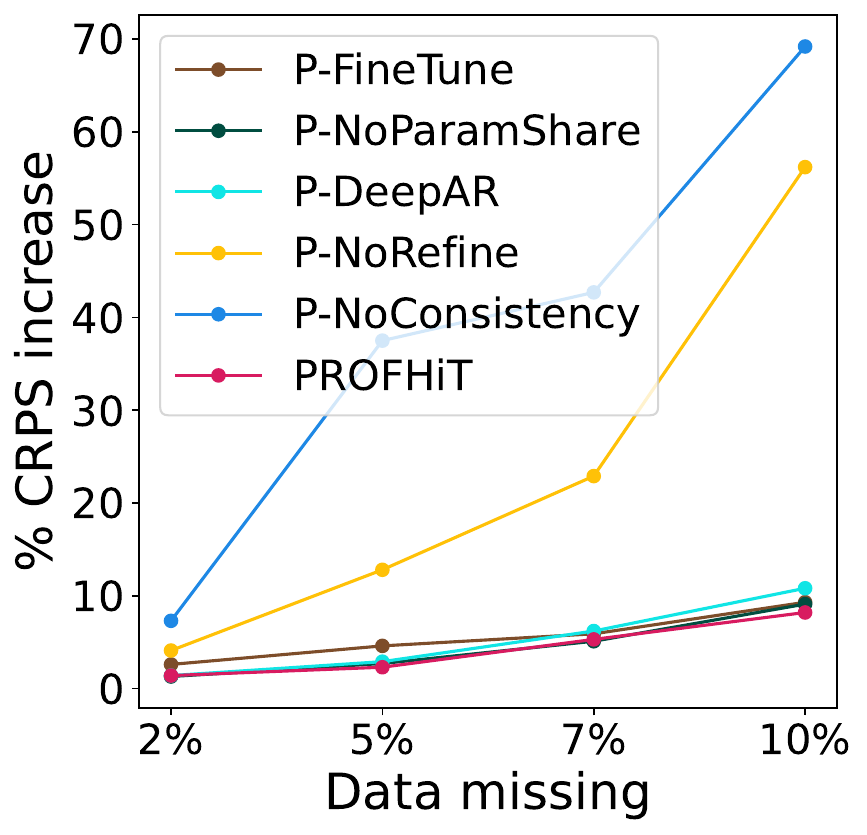}
    \caption{\labour}
  \end{subfigure}%
  \begin{subfigure}[b]{0.195\linewidth}
    \centering
    \includegraphics[width=.9\linewidth]{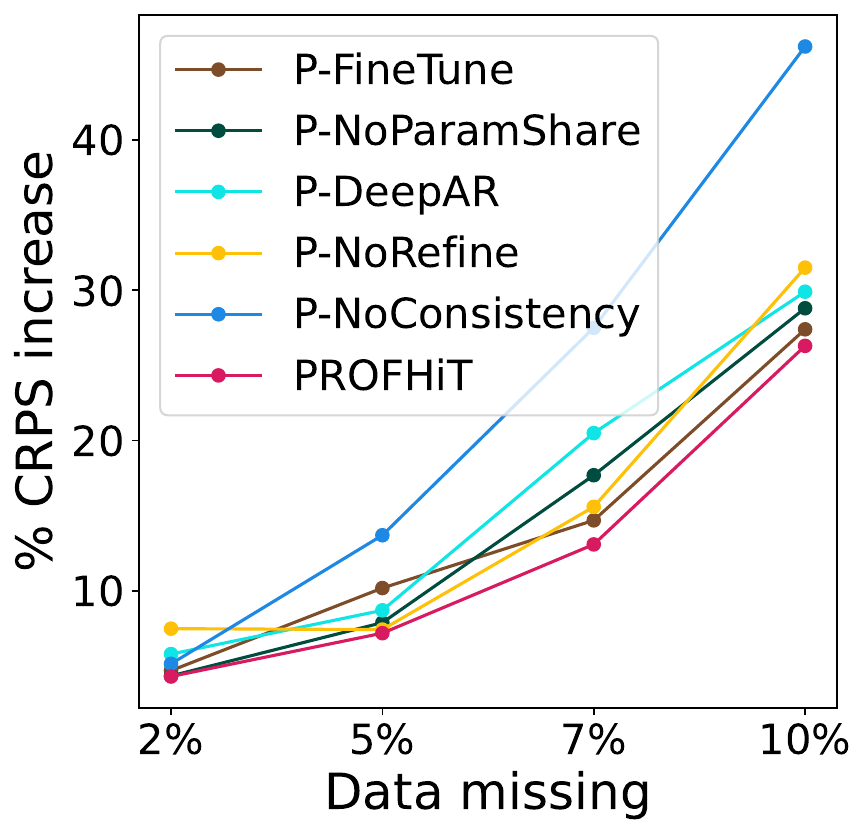}
    \caption{\wiki}
  \end{subfigure}%
  \begin{subfigure}[b]{0.195\linewidth}
    \centering
    \includegraphics[width=.9\linewidth]{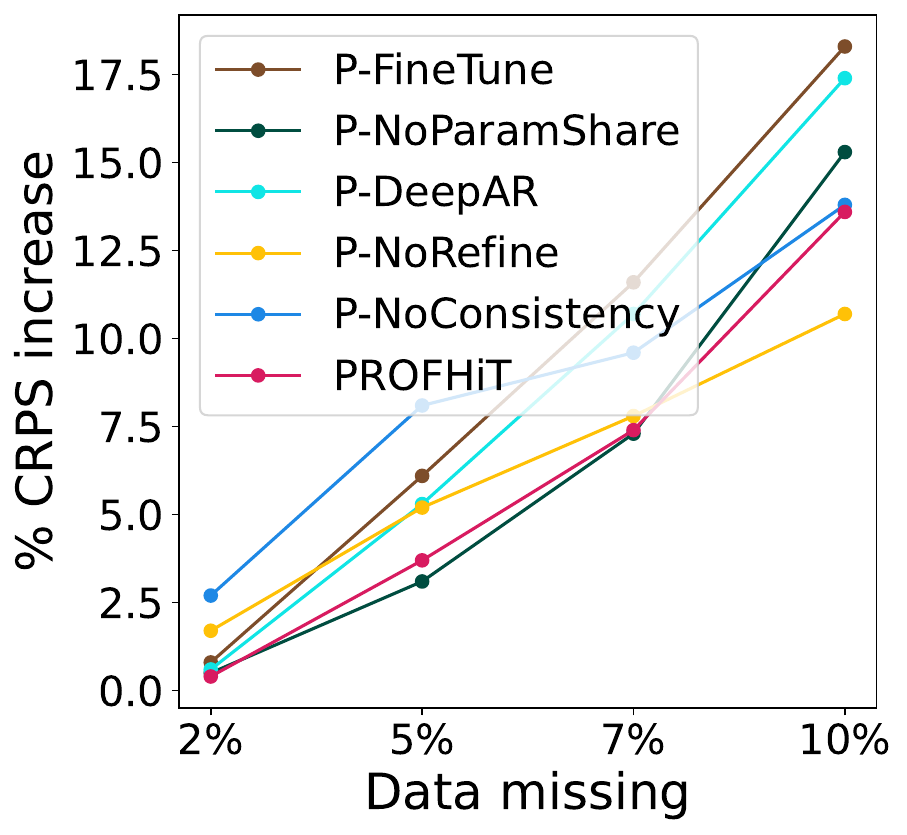}
    \caption{\symp}
  \end{subfigure}%
  \begin{subfigure}[b]{0.195\linewidth}
    \centering
    \includegraphics[width=.9\linewidth]{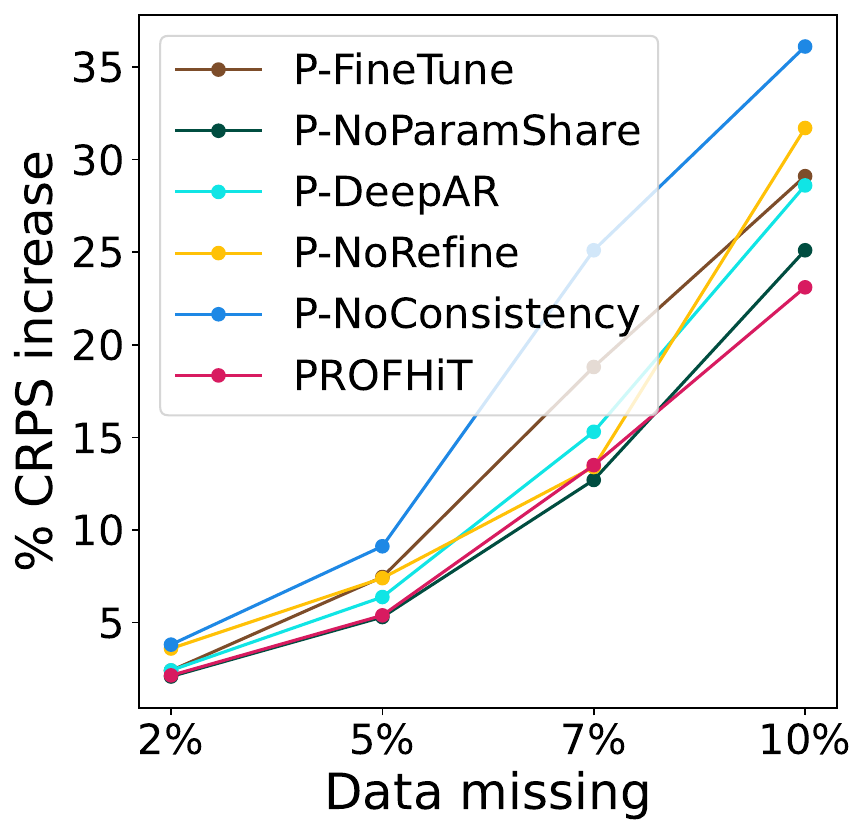}
    \caption{\fbsymp}
  \end{subfigure}
  \vspace{-0.1in}
  \caption{\% increase in CRPS for \model and variants with an increase in the proportion of missing data.}
  \label{fig:miss2}
\end{figure*}
We compare relative performance decrease with an increase in the percentage of missing data for \model and its variants in Figure \ref{fig:miss2}.
We observe that \nocoherent's performance deteriorates very rapidly in most benchmarks, showing the importance of \sdcr for learning provides calibrated and consistent forecasts.
The second worst-performing variant across all datasets is \finetune which also relies less on the hierarchical relations due to fine-tuning of parameters for specific time-series.
This is followed by \norefine which performs particularly worse in strongly consistent datasets due to the absence of the refinement module
to directly learn refined distributions by combining information from base forecasts.
Finally, we observe that \model and \pglobal suffer the least degradation in performance since both these models prioritize integrating hierarchical consistency information which enables them to provide better estimates for imputed data for missing input and use them to generate more accurate and calibrated forecasts.
\section{Conclusion and Discussion}
We introduced \model, a probabilistic hierarchical forecasting model that produces accurate and well-calibrated forecasts using soft distributional consistency regularization (\sdcr).
This enables \model to adapt to datasets with varying levels of hierarchical consistency. 
We evaluated \model against previous state-of-the-art hierarchical forecasting baselines over a wide variety of datasets and observed 41-88\% improvement average improvement in accuracy and significantly better calibration scores.
\model provided the best performance across the entire hierarchy as well as significantly outperformed other models in providing robust predictions when it encountered missing data where other baselines' performance degraded by over 70\%. 
We also showed the efficacy of various design choices of \model including using \camul for generating raw forecasts, multi-tasking approach of partial parameter sharing, refinement module, and introducing the novel distributional consistency loss as a soft regularizer.

Our work opens new possibilities like extending to various domains where time-series values across the hierarchy may not be continuous real numbers, can not be modeled as Gaussian distributions or may have different sampling rates. We can also explore modeling more complex structures between time-series with different aggregation relations.  \model can also be used to study anomaly detection in time-series, especially in time-periods where there are deviations from assumed consistency relations. Similar to \citet{kamarthi2021camul}, we can extend our work to include multiple sources of features and modalities of data both specific to each time-series and global to the entire hierarchy.
\par \noindent \textbf{Acknowledgments:} This work was supported in part by the NSF (Expeditions CCF-1918770, CAREER IIS-2028586, RAPID IIS-2027862, Medium IIS-1955883, Medium IIS-2106961, IIS-2008334, CCF-2115126, PIPP CCF-2200269, CAREER IIS-2144338), CDC MInD program, faculty research award from Facebook and funds/computing resources from Georgia Tech.


\clearpage
\bibliographystyle{ACM-Reference-Format}
\bibliography{references}

\clearpage
\appendix
\noindent{\Large \textbf{Appendix for When Rigidity Hurts: Soft Consistency Regularization for
Probabilistic Hierarchical Time Series Forecasting}}

\section{Details on \camul}
\label{sec:ngp_full}

We briefly describe the components of \camul here and direct the readers to the \cite{kamarthi2021doubt} for more details.

1) \textit{Probabilistic Neural Encoder}: It models the temporal patterns of the input time-series and the uncertainty of its latent representation.
It encodes the input univariate time-series into a latent stochastic embedding via a GRU \citep{cho2014properties} followed by a self-attention layer \citep{vaswani2017attention}:
{\begin{equation}
    \begin{split}
        [\mathbf{\mu}_{\mathbf{u}_i}, \log \sigma_{\mathbf{u}_i}] &= \text{Self-Atten}(\text{GRU}(\mathbf{y}_i^{(t':t)})),\\
        \mathbf{u}_i &\sim \mathcal{N}(\mathbf{\mu}_{\mathbf{u}_i}, \sigma_{\mathbf{u}_i}).
    \end{split}
\end{equation}}
2) \textit{Stochastic Data Correlation Graph}: We next model the correlations between the input time-series and other time-series of the dataset to capture
contextual representation and uncertainty of the input data point with respect to training data distribution.
These contextual representations are called \textbf{local latent variable}.
The only difference in our approach compared to \cite{kamarthi2021doubt} is that, unlike \cite{kamarthi2021doubt} which uses past time-series information from the same node, in our multi-variate case \camul uses past information from all nodes.
Formally, for input sequence $\mathbf{y}_i^{(t':t)}$ we sample sequences from the the past training sequences $\mathbf{y}_j$ where $j \in \{1, \dots, N\}$ using similarity between their latent stochastic embeddings $\{\mathbf{u}_i\}_{i=1}^N$.
For input time-series of node $i$ and each of the past training sequences $\mathbf{y}_j$, we sample $\mathbf{y_j}$  with probability $\exp(-\gamma||\mathbf{u}_i-\mathbf{u}_j||_2^2)$ into set $N_i$. Then, we derive the local latent variable as
    {\begin{equation}
            \mathbf{z}_i \sim \mathcal{N}\left( \sum_{j\in N_i}\Theta_1(\mathbf{u}_j), \exp(\sum_{j\in N_i}\Theta_2(\mathbf{u}_j)) \right)
        \end{equation}}
where $\Theta_1$ and $\Theta_2$ are feed-forward networks.

3) \textit{Predictive Distribution Decoder}: The final step of \camul's stochastic process involves combining the representations from Probabilistic Neural Encoder and Stochastic Data Correlation Graph that capture relevant sequential and contextual representation and uncertainty of input time-series.
We combine the latent stochastic embedding, local latent variable and combined information of all past sequences to derive the parameters of the output distribution via a simple feed-forward network.
We first derive a \textit{global latent variable} that combines the information from local latent embeddings of all past sequences as $\mathbf{z} = \text{Self-Atten}(\{\mathbf{u}_i\}_{i=1}^N)$ via a self-attention layer over $\{\mathbf{u}_i\}_{i=1}^N$ and summation of self-attention layer's output.

Finally, we combine the latent embedding of input time-series, local latent variable and global latent variable to derive the base forecast distribution modeled as a Gaussian $\mathcal{N}(\basemu_i, \basesigma_i)$ as:
{\begin{equation}
    \begin{split}
        \mathbf{e} = \text{concat}(\mathbf{u}_i, \mathbf{z}_i, \mathbf{z}), \quad
        [\basemu_i, \log \basesigma_i] = \Theta_3(\mathbf{e})
    \end{split}
\end{equation}}
where $\Theta_3$ is a feed-forward network.

The full stochastic process of \camul can be summarized as:

{
\begin{equation}
    \begin{split}
        P&(\{\basemu_i, \basesigma_i\}_{i=1}^N| \mathcal{D}^t) =
        \int
        \underbrace{\left(\prod_{i=1}^N P(\mathbf{u}_i | \mathbf{y}_i^{(1:t)})\right)}_{\text{Probabilistic Encoder}}\\
        &\underbrace{\left(\prod_{i=1}^NP(N_i|\{\mathbf{u}_i\}_{i=1}^N)  P(\mathbf{z}_i|N_i, \{\mathbf{u}_j\}_{j=1}^N)\right)}_{\text{SDCG}}
        \underbrace{P(\mathbf{z}|\{\mathbf{u}_i\}_{i=1}^N)}_{\text{Global Latent variable}}\\
        &\underbrace{\left(\prod_{i=1}^NP(\basemu_i, \basesigma_i | \mathbf{z}, \mathbf{z}_i, \mathbf{u}_i)\right)}_{\text{Raw forecasts}} d\{\mathbf{u}_i\}_{i=1}^N d\{\mathbf{z}_i\}_{i=1}^N d\{N_i\}_{i=1}^N.
    \end{split}
    \label{eqn:process1}
\end{equation}
}
Note that in the main paper we note the set of all latent variables $\{\mathbf{u}_i, \mathbf{z}_i, N_i\}_{i=1}^N, \mathbf{z}$ as $\mathscr{Z}$.

\paragraph{Note on running time}
The novel component of \model is the Hierarchy-aware refinement module that facilitates the integration of base forecast distributions. As described in lines 398-403, the total computational complexity of obtaining the refined distributional parameters is $O(N^2)$ ($N$ is the number of nodes in the hierarchy), which is comparable to the reconciliation step of end-to-end methods like \hiere. Post-processing techniques such as \mint, ERM, and PEMBU have an even higher time complexity of $O(N^3)$.

Note that the other portion of the pipeline that may add to the time-complexity is the base forecasting models. Models like \deepvar and RNN used by \hiere, SHARQ as well as the post-processing methods and TSFNP (used by \model) scale linearly with respect to the length of the time-series and linearly with the number of nodes $N$. Therefore all these baselines and \model use models with similar time-complexity for base forecasts with respect to the size of the hierarchy $N$.

\section{Code and Dataset}
We evaluated all models on a system with Intel 64-core Xeon Processor with 128 GB memory and Nvidia Tesla V100 GPU with 32 GB VRAM.
We provide our implementation of \model along with the datasets used at \url{https://github.com/AdityaLab/Profhit}.

\section{Hyperparameters}

\subsection{Data Preprocessing}
Most datasets used in our work assume the aggregation function to be a simple summation (i.e, $\phi_{ij}=1$ for all weights). We first normalize the values of leaf time-series training data to have 0 mean and variance of 1.
Since the aggregation of values at higher levels of the hierarchy can lead to
very large values in time-series, we instead divide each non-leaf time-series by the number of children. Then the weights of hierarchical relations become $\phi_{ij} = \frac{1}{|C_i|}$ where $C_i$ is the set of all children nodes of time-series $i$.
For the remaining datasets (Flu-Symptoms, FB-Symptoms) the time-series values are normalized by default and thus require no extra pre-processing.

\subsection{Model Architecture}
The architecture of \camul used in \model is similar to that used in the original implementation \citep{kamarthi2021doubt}. The GRU unit contains 60 hidden units and is bi-directional. Thus the local latent variable is also of dimension 60. $NN_1$ and $NN_2$ are both 2-layered neural networks with the first layer shared between both. Both layers have 60 hidden units. Finally, $NN_3$ is a three-layer neural network with the input layer having $180$ units (for the concatenated input of three 60-dimensional vectors) and the last two layers having 60 hidden units. We found that the value of $c$ in Equation \ref{eqn:corem2} is not very sensitive and is usually set to 5.

Note that we do not explicitly model covariance between every pair of time series (like \mint, \erm) and use a weighted combination of base forecast parameters to derive refined forecasts. Therefore the refinement module complexity (Section \ref{sec:refinement}) is $O(N^2)$ which is on par with previous methods like \hiere.

\subsection{Training and Evaluation}

Given the training dataset $\mathcal{D}_t$ we extract the training dataset for each node as the set of prefix sequences $\{(\mathbf{y}_i^{(t1:t2)}, y_{i}^{(t2+1)}): 1\leq t1 \leq t2 < t-\tau\}$ and train the full model (\camul and refinement module). We tune the hyperparameter using backtesting by validating on window $t-\tau$ to $t$. Finally, we train for entire training set with the best hyperparameters.

For each benchmark, we used the validation set to mainly find the optimal batch size and learning rate. We searched over batch-size of $\{10, 50, 100, 200\}$ and the optimal learning rate was usually around 0.001. We also found the optimal $\lambda$ to be around $0.01$ for strongly consistent datasets and $0.001$ for weakly consistent datasets.
We used early stopping with the patience of 150 epochs to prevent overfitting. For each independent run of a model, we initialized the random seeds from 0 to 5 for PyTorch and NumPy. We didn't observe large variations due to randomness for \model and all baselines.

During evaluation, we sampled 2000 Monte-Carlo samples of the forecast distribution
and used it to estimate the mean for MAPE. We also used the samples mean and variance to evaluate LS and CS whereas used ensemble scoring to evaluate CRPS directly from the samples using \texttt{properscoring} package \footnote{\url{https://github.com/properscoring/properscoring}}.

\section{Derivation of Likelihood ELBO loss}
The full predictive distribution of \model from Equation \ref{eqn:process} can be further expanded as:
\begin{equation}
    \begin{split}
        P(&\{y_i^{(t+\tau)}\}_{i=1}^N| \mathcal{D}_t) =
        \int
        \underbrace{\left(\prod_{i=1}^N P(\mathbf{u}_i | \mathbf{y}_i^{(1:t)})\right)}_{\text{Probabilistic Encoder}}\\
        &\underbrace{\left(\prod_{i=1}^NP(N_i|\{\mathbf{u}_i\}_{i=1}^N)  P(\mathbf{z}_i|N_i, \{\mathbf{u}_j\}_{j=1}^N)\right)}_{\text{SDCG}}\\
        &\underbrace{P(\mathbf{z}|\{\mathbf{u}_i\}_{i=1}^N)}_{\text{Global Latent variable}}
        \underbrace{\left(\prod_{i=1}^NP(\basemu_i, \basesigma_i | \mathbf{z}, \mathbf{z}_i, \mathbf{u}_i)\right)}_{\text{Raw forecasts}}\\
        &\underbrace{\prod_{i=1}^N P(\refinedmu_i, \refinedsigma_i| \{\basemu_j, \basesigma_j\}_{j=1}^N) P(y_i^{(t+\tau)}|\refinedmu_i, \sigma_i)}_{\text{Refinement Module}} d\{\mathbf{u}_i\}_{i=1}^N d\{\mathbf{z}_i\}_{i=1}^N.
    \end{split}
    \label{eqn:process_2}
\end{equation}

To minimize the data likelihood $P(\{y_i^{(t+\tau)}\}_{i=1}^N| \mathcal{D}_t)$ requires intregration over latent variables $\{\mathbf{u}_i\}_{i=1}^N$ and $\{\mathbf{z}_i\}_{i=1}^N$. We instead perform amortized variational inference on the latent variables similar to VAE \citep{kingma2013auto}.

We approximate the posterior of latent variables \[P(\{\mathbf{u}_i\}_{i=1}^N, \{\mathbf{z}_i\}_{i=1}^N, \{N_i\}_{i=1}^N| \{y_i^{(t+\tau)}\}_{i=1}^N)\] with a variational distribution \[Q(\{\mathbf{u}_i\}_{i=1}^N, \{\mathbf{z}_i\}_{i=1}^N, \{N_i\}_{i=1}^N| \{y_i^{(t+\tau)}\}_{i=1}^N)\] expressed as:

\begin{equation}
    \begin{split}
        Q&(\{\mathbf{u}_i\}_{i=1}^N, \{\mathbf{z}_i\}_{i=1}^N, \{N_i\}_{i=1}^N| \{y_i^{(t+\tau)}\}_{i=1}^N) = \left(\prod_{i=1}^N P(\mathbf{u}_i | \mathbf{y}_i^{(1:t)})\right) \\
        &\left(\prod_{i=1}^NP(N_i|\{\mathbf{u}_i\}_{i=1}^N)\right)
        \left(\prod_{i=1}^N q_{\phi}(\mathbf{z}_i|\mathbf{y}_i^{(1:t)})\right)
    \end{split}
\end{equation}
where $q_{\phi}$ is a feed-forward network over GRU embeddings of Probabilistic Neural Encoder that parameterizes to a gaussian distribution of $\mathbf{z}_i$.

The ELBO loss
\begin{equation}
    \begin{split}
        -\mathbb{E}_{Q(\{\mathbf{u}_i, \mathbf{z}_i, N_i\}_{i=1}^N| \{y_i^{(t+\tau)}\}_{i=1}^N)} \\
        [ \log P(\{y_i^{(t+\tau)}\}_{i=1}^N| \{\mathbf{u}_i, \mathbf{z}_i, N_i\}_{i=1}^N) \\
        + \log P(\{\mathbf{u}_i\}_{i=1}^N, \{\mathbf{z}_i\}_{i=1}^N, \{N_i\}_{i=1}^N| \{y_i^{(t+\tau)}\}_{i=1}^N) \\
        - \log Q(\{\mathbf{u}_i\}_{i=1}^N, \{\mathbf{z}_i\}_{i=1}^N, \{N_i\}_{i=1}^N| \{y_i^{(t+\tau)}\}_{i=1}^N) ]
    \end{split}
\end{equation}
get simplified to:
\begin{equation}
    \begin{split}
        \mathcal{L}_{1} =& -E_{Q(\{\mathbf{u}_i, \mathbf{z}_i, N_i\}_{i=1}^N, \mathbf{z}| \{y_i^{(t+\tau)}\}_{i=1}^N)} [\log P(\{y_i^{(t+\tau)}\}_{i=1}^N| \{\mathbf{u}_i, \mathbf{z}_i, N_i\}_{i=1}^N)\\
        &+ \sum_{i=1}^N \log P(\mathbf{z}_i | \{\mathbf{u}_j\}_{j=1}^N, N_i) - \log q_i(\mathbf{z}_i | \mathbf{y}_i^{(t':t)})].
    \end{split}
    \label{eqn:loss1}
\end{equation}
by canceling similar terms between the variational and true distribution of latent variables.

\section{Consistency of datasets}
\label{sec:deviation}
We noted in Section \ref{sec:results} Q4 that \texttt{Flu-Symptoms} and \texttt{FB-Survey} are weakly consistent datasets since they do not strictly follow the aggregation relations $H_{\mathcal{T}}$ unlike strongly consistent datasets \texttt{Tourism-L, Labour, Wiki}.

\begin{table}[h]
    \caption{\textit{Average deviation of observed values in time-series from hierarchical relations.}}
    \label{tab:deviation}
    \centering
    \scalebox{0.85}{
        \begin{tabular}{c|ccccc}
            \textbf{Data}            & \texttt{Flu} & \texttt{FB-Survey} & \texttt{Tourism-L} & \texttt{Labour} & \texttt{Wiki} \\ \hline
            Level 1                  & 0.043        & 1.27               & 0                  & 0               & 0             \\
            Level 2                  & 3.41         & 2.83               & 0                  & 0               & 0             \\
            Average across hierarchy & 3.37         & 2.44               & 0                  & 0               & 0
        \end{tabular}
    }
\end{table}

We empirically observe this by measuring Consistency errors of all datasets (Definition \ref{def:ce}) for the entire hierarchy and at each level of the hierarchy.
The results are in Table \ref{tab:deviation}. As expected there are no deviations for strongly consistent datasets whereas there is a significant deviation in weakly consistent data.

\section{Performance across each level of hierarchy}
We compared the performance of \model with best-performing baselines \hiere and \sharq for each level of hierarchy of all datasets.
\model significantly outperforms the best baselines as well as the variants. At the leaf nodes, which contains most data, \model outperforms best baselines by 7\% in \wiki to 100\% in \fbsymp. For the top node of time-series the performance improvement is largest at 35\% (\wiki) to 962\% (\fbsymp). We show detailed results in Tables \ref{tab:levels} and \ref{tab:levelscs}.

\begin{table*}[h]
\caption{\textit{Average CRPS scores at each level of hierarchy. \model significantly outperforms best baselines across all benchmarks. Note that P-Finetune's performance decreases at higher levels of hierarchy compared to other variants whereas P-Global's performance is worse at lower levels.}}
\centering
\scalebox{0.95}{
\begin{tabular}{c|cccccccccccc}
Models/Data                    & \multicolumn{8}{c}{\tourism}                                                                                                                                                                                                                                                                                                                & \multicolumn{4}{|c}{\labour}                                                                   \\ \hline
     Hierarchy Levels                     & 1                                  & 2(Travel)                                  & 3(Travel)                                  & 4(Travel)                                  & 5(Travel)                                   & 2(Geo)                             & 3(Geo)                             & 4(Geo)                              & \multicolumn{1}{|c}{1}              & 2                                 & 3                                 & 4                    \\ \hline
\hiere                   & 0.081                              & 0.103                              & 0.141                              & 0.205                              & 0.272                               & 0.103                              & 0.136                              & 0.175                               & \multicolumn{1}{|c}{0.031}          & 0.034                             & 0.034                             & 0.038                \\
\sharq                     & 0.093                              & 0.131                              & 0.163                              & 0.218                              & 0.295                               & 0.131                              & 0.138                              & 0.152                               & \multicolumn{1}{|c}{0.097}          & 0.124                             & 0.133                             & 0.149                \\ 
\pembu-\mint                     & 0.112                              & 0.121                              & 0.139                              & 0.203                              & 0.185                               & 0.116                              & 0.128                              & 0.167                               & \multicolumn{1}{|c}{0.063}          & 0.033                             & 0.042                             & 0.085                \\ \hline
\model(Ours)                      & \textbf{0.051}                     & \textbf{0.095}                     & \textbf{0.12}                      & \textbf{0.17}                      & \textbf{0.264}                      & \textbf{0.083}                     & \textbf{0.106}                     & \textbf{0.142}                      & \multicolumn{1}{|c}{\textbf{0.023}} & \textbf{0.019}                    & \textbf{0.023}                    & \textbf{0.029}       \\ \hline
\finetune   & 0.072          & 0.136          & 0.083          & 0.16           & 0.278                               & 0.124          & 0.124          & 0.158                               & \multicolumn{1}{|c}{0.024}          & 0.022          & 0.026          & 0.035          \\
\pglobal     & 0.093          & 0.113          & 0.122          & 0.13           & \textbf{0.261}                      & 0.093          & 0.113          & 0.147                               & \multicolumn{1}{|c}{\textbf{0.021}} & 0.027          & 0.028          & \textbf{0.027} \\
\pvar & 0.075          & 0.097          & 0.136          & 0.183          & 0.281                               & 0.095          & 0.122          & 0.159                               & \multicolumn{1}{|c}{0.025}          & 0.027          & 0.031          & 0.033          \\
\nocoherent    & 0.086          & 0.142          & 0.107          & 0.18           & 0.265                               & 0.132          & 0.138          & 0.147                               & \multicolumn{1}{|c}{0.027}          & 0.031          & 0.029          & 0.026          \\ \hline
Models/Data                          & \multicolumn{5}{c|}{\wiki}                                                                                                                                                               & \multicolumn{3}{c|}{\symp}                                                                                      & \multicolumn{3}{c}{\fbsymp}                                                                             &                      \\ \hline
 Hierarchy Levels                         & 1                                  & 2                                  & 3                                  & 4                                  & \multicolumn{1}{c|}{5}              & 1                                  & 2                                  & \multicolumn{1}{c|}{3}              & 1                                   & 2                                 & 3                                 &                      \\ \hline
\hiere                   & 0.042                              & 0.105                              & 0.229                              & 0.272                              & \multicolumn{1}{c|}{0.372}          & 0.272                              & 0.421                              & \multicolumn{1}{c|}{0.458}          & 4.14                                & 4.04                              & 4.13                              &                      \\
\sharq                     & 0.039                              & 0.136                              & 0.235                              & 0.291                              & \multicolumn{1}{c|}{0.378}          & 0.258                              & 0.376                              & \multicolumn{1}{c|}{0.381}          & 3.08                                & 3.21                              & 3.13                              &                      \\ 
\pembu-\mint                    & 0.031                              & 0.171                              & 0.241                              & 0.385                              & \multicolumn{1}{c|}{0.433}          & 0.337                              & 0.567                              & \multicolumn{1}{c|}{0.773}          & 4.82                                & 5.53                              & 6.15                              &                      \\ \hline
\multicolumn{1}{l|}{\model (Ours)} & \multicolumn{1}{r}{\textbf{0.031}} & \multicolumn{1}{c}{\textbf{0.074}} & \multicolumn{1}{c}{\textbf{0.133}} & \multicolumn{1}{c}{\textbf{0.216}} & \multicolumn{1}{c|}{\textbf{0.252}} & \multicolumn{1}{c}{\textbf{0.216}} & \multicolumn{1}{c}{\textbf{0.133}} & \multicolumn{1}{c|}{\textbf{0.338}} & \multicolumn{1}{c}{\textbf{0.32}}   & \multicolumn{1}{c}{\textbf{0.43}} & \multicolumn{1}{c}{\textbf{1.89}} & \multicolumn{1}{l}{}\\ \hline
\finetune   &0.034          & 0.086          & 0.153          & 0.232          & \multicolumn{1}{c|}{0.275} & 0.222          & 0.175          & \multicolumn{1}{c|}{\textbf{0.293}} & 0.43                                & 0.65           & \textbf{1.83}  &                \\
\pglobal     &  0.048          & 0.103          & 0.187          & 0.265          & \multicolumn{1}{c|}{\textbf{0.186}}         & 0.269          & 0.213          & \multicolumn{1}{c|}{0.376}          & 0.37                                & \textbf{0.37}  & 2.11           &                \\
\pvar & 0.035          & 0.094          & 0.193          & 0.251          & \multicolumn{1}{c|}{0.285}          & 0.242          & 0.217          & \multicolumn{1}{c|}{0.328}          & 0.44                                & 0.61           & 2.01           &                \\
\nocoherent    & 0.49           & 0.117          & 0.93           & 0.258          & \multicolumn{1}{c|}{0.167}          & 0.227          & 0.193          & \multicolumn{1}{c|}{0.381}          & 0.42                                & \textbf{0.36}  & 2.18           &            \\   
\end{tabular}
}
\label{tab:levels}
\end{table*}

\begin{table*}[h]
\caption{\textit{Average CS scores at each level of hierarchy. \model significantly outperforms best baselines across all benchmarks. }}
\centering
\scalebox{0.95}{
\begin{tabular}{c|cccccccc|cccc}
Models/Data      & \multicolumn{8}{c|}{\tourism}                                                                                                                        & \multicolumn{4}{c}{\labour}                                    \\ \hline
Hierarchy Levels & 1             & 2             & 3             & 4             & 5                                  & 2(Geo)         & 3(Geo)         & 4(Geo)         & 1             & 2             & 3             & 4             \\ \hline
\hiere          & 0.15          & 0.18          & 0.17          & 0.21          & 0.24                               & 0.19           & 0.18           & 0.22           & 0.21          & 0.23          & 0.22          & 0.27          \\
SHARQ            & 0.09          & 0.08          & 0.12          & 0.11          & 0.14                               & 0.11           & 0.12           & 0.16           & 0.16          & 0.16          & 0.15          & 0.21          \\
PEMBU-\mint       & 0.14          & 0.21          & 0.22          & 0.21          & 0.26                               & 0.18           & 0.23           & 0.25           & 0.21          & 0.22          & 0.24          & 0.21          \\\hline
\model          & \textbf{0.05} & \textbf{0.06} & \textbf{0.04} & \textbf{0.06} & \textbf{0.11}                      & \textbf{0.06}  & \textbf{0.06}  & \textbf{0.1}   & \textbf{0.17} & \textbf{0.11} & \textbf{0.15} & \textbf{0.16} \\\hline
\finetune       & 0.09          & 0.12          & 0.13          & 0.17          & 0.13                               & 0.11           & 0.13           & 0.15           & 0.24          & 0.21          & 0.24          & 0.22          \\
\pglobal         & 0.06          & \textbf{0.04} & \textbf{0.03} & 0.08          & \textbf{0.05}                      & \textbf{0.05}  & \textbf{0.03}  & \textbf{0.04}  & \textbf{0.14} & 0.18          & 0.19          & \textbf{0.15} \\
\pvar         & 0.11          & 0.09          & 0.09          & 0.14          & 0.13                               & 0.15           & 0.14           & 0.13           & 0.14          & 0.19          & 0.17          & 0.14          \\
\nocoherent     & 0.18          & 0.19          & 0.17          & 0.19          & 0.22                               & 0.18           & 0.19           & 0.24           & 0.24          & 0.22          & 0.25          & 0.31          \\ \hline
Models/Data      & \multicolumn{5}{c|}{\wiki}                                                                          & \multicolumn{3}{c|}{\symp}                         & \multicolumn{3}{c}{\fbsymp}               &               \\ \hline
Hierarchy Levels & 1             & 2             & 3             & 4             & \multicolumn{1}{c|}{5}             & 1              & 2              & 3              & 1             & 2             & 3             &               \\ \hline
\hiere          & 0.15          & 0.21          & 0.26          & 0.22          & \multicolumn{1}{c|}{0.24}          & 0.11           & 0.13           & 0.11           & 0.21          & 0.19          & 0.18          &               \\
SHARQ            & 0.13          & 0.14          & 0.14          & 0.17          & \multicolumn{1}{c|}{0.15}          & 0.58           & 0.052          & 0.085          & 0.16          & 0.14          & 0.15          &               \\
PEMBU-\mint       & 0.12          & 0.11          & 0.12          & 0.13          & \multicolumn{1}{c|}{0.14}          & 0.17           & 0.22           & 0.17           & 0.2           & 0.19          & 0.16          &               \\ \hline
\model          & \textbf{0.11} & \textbf{0.15} & \textbf{0.12} & \textbf{0.14} & \multicolumn{1}{c|}{\textbf{0.11}} & \textbf{0.031} & \textbf{0.044} & \textbf{0.052} & \textbf{0.09} & \textbf{0.07} & \textbf{0.06} &               \\ \hline
\finetune       & 0.19          & 0.18          & 0.23          & 0.22          & \multicolumn{1}{c|}{0.24}          & 0.033          & \textbf{0.031} & \textbf{0.042} & \textbf{0.05} & \textbf{0.06} & 0.09          &               \\
\pglobal         & 0.16          & \textbf{0.15} & 0.16          & 0.17          & \multicolumn{1}{c|}{0.15}          & 0.065          & 0.072          & 0.096          & 0.11          & 0.13          & 0.17          &               \\
\pvar         & 0.21          & 0.24          & 0.26          & 0.22          & \multicolumn{1}{c|}{0.23}          & 0.064          & 0.077          & 0.083          & 0.15          & 0.19          & 0.17          &               \\
\nocoherent     & 0.29          & 0.28          & 0.35          & 0.33          & \multicolumn{1}{c|}{0.37}          & 0.22           & 0.18           & 0.14           & 0.22          & 0.25          & 0.21          &              
\end{tabular}
}
\label{tab:levelscs}
\end{table*}
\begin{table*}[h]
    \caption{\textit{\textbf{Std. dev} of CRPS and LS (accros 5 runs) across all levels for all baselines, \model and its variants. \model performs significantly better than all baselines as noted using t-test with $\alpha=1\%$.}}
    \centering
    \scalebox{0.95}{
        \begin{tabular}{cc|cc|cc|cc|cc|cc}
                               & Models/Data   & \multicolumn{2}{c|}{\tourism} & \multicolumn{2}{c|}{\labour} & \multicolumn{2}{c|}{\wiki} & \multicolumn{2}{c|}{\symp} & \multicolumn{2}{c}{\fbsymp}                                         \\ \hline
                               &               & CRPS                          & LS                           & CRPS                       & LS                         & CRPS                        & LS    & CRPS  & LS    & CRPS  & LS    \\ \hline
            \textbf{Baselines} & \deepvar      & 0.011                         & 0.040                        & 0.004                      & 0.038                      & 0.002                       & 0.044 & 0.018 & 0.098 & 0.482 & 0.434 \\
                               & \camul        & 0.006                         & 0.021                        & 0.003                      & 0.018                      & 0.015                       & 0.069 & 0.019 & 0.004 & 0.251 & 0.217 \\
                               & \mint         & 0.005                         & 0.019                        & 0.002                      & 0.121                      & 0.018                       & 0.006 & 0.014 & 0.111 & 0.468 & 0.213 \\
                               & \erm          & 0.044                         & 0.005                        & 0.002                      & 0.110                      & 0.016                       & 0.069 & 0.018 & 0.133 & 0.148 & 0.209 \\
                               & \hiere        & 0.001                         & 0.038                        & 0.003                      & 0.049                      & 0.019                       & 0.018 & 0.005 & 0.051 & 0.325 & 0.109 \\
                               & \sharq        & 0.000                         & 0.011                        & 0.001                      & 0.046                      & 0.017                       & 0.007 & 0.002 & 0.116 & 0.133 & 0.048 \\ \hline
                               & \model (Ours) & 0.001                         & 0.017                        & 0.001                      & 0.003                      & 0.001                       & 0.030 & 0.005 & 0.009 & 0.040 & 0.008 \\ \hline
            \textbf{Ablation}  & \finetune     & 0.016                         & 0.031                        & 0.003                      & 0.003                      & 0.016                       & 0.014 & 0.001 & 0.006 & 0.090 & 0.004 \\
                               & \pglobal      & 0.012                         & 0.033                        & 0.000                      & 0.013                      & 0.002                       & 0.001 & 0.033 & 0.024 & 0.248 & 0.119 \\
                               & \pvar         & 0.006                         & 0.026                        & 0.001                      & 0.028                      & 0.005                       & 0.043 & 0.035 & 0.030 & 0.103 & 0.065 \\
                               & \nocoherent   & 0.005                         & 0.012                        & 0.001                      & 0.009                      & 0.015                       & 0.043 & 0.012 & 0.025 & 0.110 & 0.053
        \end{tabular}
    }
    \label{tab:stdev}
\end{table*}
\clearpage

\section{Details on Data Imputation Experiment}

\paragraph{Motivation:} During real-time forecasting in real-world applications such as Epidemic or Sales forecasting, we encounter situations where the past few values of time-series are missing or unreliable for some of the nodes. This is observed specifically at lower levels, due to discrepancies or delays during reporting and other factors \citep{chakraborty2018know}.
Therefore, one approach to performing forecasting in such a situation is first by imputation of missing values based on past data and then using the predicted missing values as part of the input for forecasting.

\paragraph{Task:}  To simulate such scenarios of missing data and evaluate the robustness of \model and all baselines, we design a task called \textit{Hierarchical Forecasting with Missing Values} (HFMV). Formally, at time-period $t$, we are given full data for up to time $t-\rho$.
We show results here for $\rho=5$ which is the average forecast horizon of all tasks.
For sequence values in the time period between $t-\rho$ and $t$, we randomly remove $k\%$ of these values across all time-series. The goal of HFMV task is to use the given partial dataset from $t-\rho$ to $t$ as input along with complete dataset for time-period before $t-\rho$ to predict future values at $t+\tau$. Therefore, success in HFMV implies that models are robust to missing data from the recent past by effectively leveraging hierarchical relations.

\paragraph{Setup:} We first train \model and baselines on complete dataset till time $t'$ and then fill in the missing values of input sequence using the trained model. Using the predicted missing values, we again forecast the output distribution. For each baseline and \model, we perform multiple iterations of Monte-Carlo sampling for missing values followed by forecasting future values to generate the forecast distribution.
We estimate the evaluation scores using sample forecasts from all sampling iterations.

\section{Adapting to varying dataset consistency}
\label{sec:cases}

\begin{observation}
    The average improvement in performance of \model over best forecasting baselines is 72\% higher for weakly consistent datasets over its improvement for strongly consistent datasets.
\end{observation}
Since most previous state-of-art models assume datasets to be strongly consistent, deviations
to this assumptions can cause  under-performance when used with weakly consistent datasets.
This is evidenced in Table \ref{tab:tab1} where some of the baselines like \mint and \erm
that explicitly optimize for hierarchical consistency perform worse than even \camul, which does not leverage hierarchical relations, in \symp and \fbsymp.
Overall, we found that for weakly consistent datasets, \model provides a much larger 93\% average improvement in CRPS scores over the best baselines
compared to 54\% average improvement for strongly consistent datasets.
These improvements are more pronounced at non-leaf nodes of hierarchy where \model improves by 2.8 times for \symp
and 9.2 times for \fbsymp. This is because the baselines which assume strong consistency do not adapt to
noise at leaf nodes that compound to errors at higher levels of hierarchy.

\begin{observation}
    \model's approach to parameter sharing and soft consistency regularization helps adapt to varying hierarchical consistency.
\end{observation}
We observe that that best performing variant for strongly consistent datasets in \pglobal which is trained with both likelihood loss and \sdcr (Table \ref{tab:tab1}). But its performance severely degrades for weakly consistent datasets since sharing all model parameters across all time-series makes it inflexible to model patterns and deviations specific to individual nodes. In contrast, \finetune and \nocoherent performs the best among variants for weakly consistent datasets since they train separate sets of decoder parameters for each node. But they perform poorly for strongly consistent datasets since they don't leverage Distributional Consistency effectively. \model combines the  flexible parameter learning of \finetune and leverage Distributional Consistency to jointly optimize the parameters like \pglobal providing comparable performance to best variants over all datasets.

\begin{table}[h]
    \caption{Average value of $\gamma_i$ for all datasets. Note that weakly consistent datasets have higher $\gamma_i$ (depends mode on past data of same time-series) where as strongly-consistent data have lower $\gamma_i$ (leverages the hierarchical relations).}
    \centering
    \scalebox{0.99}{
        \begin{tabular}{c|cc}
            Consistency & Dataset   & \multicolumn{1}{l}{Average value of $\gamma_i$} \\ \hline
            Strong      & Tourism-L & 0.420  $\pm$ 0.096                              \\
                        & Labour    & 0.348  $\pm$ 0.091                              \\
                        & Wiki      & 0.313  $\pm$ 0.057                              \\ \hline
            Weak        & Symp      & 0.759  $\pm$ 0.152                              \\
                        & Fbsymp    & 0.789  $\pm$ 0.180
        \end{tabular}
    }
    \label{tab:gamma}
\end{table}

\begin{observation}
    \model's Refinement module automatically learns to adapt to varying hierarchical consistency.
\end{observation}
The design choices of the refinement module  help \model to adapt to datasets of different levels of hierarchical consistency.
Specifically, by optimizing for values of $\{\gamma_i\}_{i=1}^N$ of Equation \ref{eqn:corem1}, \model aims to learn a good trade-off between leveraging prior forecasts for a time-series and hierarchical relations of forecasts from the entire hierarchy.
We study the learned values of $\{\gamma_i\}_{i=1}^N$ of Equation \ref{eqn:corem1} used to derive refined mean. Note that higher values of $\gamma_i$ indicate larger dependence on base forecasts of node and smaller dependence of forecasts of the entire hierarchy.
We plot the average values of $\gamma_i$ for each of the datasets in Table \ref{tab:gamma}. We observe that strongly consistent datasets have lower values of $\gamma_i$ indicating that \model's refinement module automatically learns to strongly leverage the hierarchy for these datasets compared to weakly consistent datasets.

\end{document}